\documentclass[journal, lettersize]{IEEEtran}
\usepackage{dsfont, adjustbox}
\usepackage{subfiles}
\usepackage{cuted}
\usepackage{xr}
\usepackage{algpseudocode}
\usepackage{tablefootnote}
\usepackage{amsthm, xpatch, bm} 
\usepackage[thinc]{esdiff}
\usepackage[bookmarks=false]{hyperref} 
\hypersetup{
     colorlinks = true,
     linkcolor = blue,
     anchorcolor = black,
     citecolor = blue,
     filecolor = blue,
     urlcolor = black,
}
\usepackage[justification=centering]{caption}
\usepackage{subcaption}
\usepackage[font=small,skip=0pt]{subcaption}
\usepackage{url}
\usepackage{cite}
\usepackage{xcolor}
\usepackage[nolist]{acronym}
\usepackage{graphicx, graphics} 
\usepackage{float}
\usepackage{threeparttable}
\usepackage{cancel}
\usepackage[official]{eurosym}
\usepackage{fancyhdr}
\usepackage{comment}
\usepackage{chemformula}
\usepackage{setspace}
\usepackage[T1]{fontenc} 
\usepackage{amsmath, mathtools, tcolorbox}
\usepackage{mathptmx}
\usepackage{cleveref}
\usepackage[cmintegrals]{newtxmath}
\usepackage{bm} 
\usepackage[colorinlistoftodos,prependcaption,textsize=tiny]{todonotes}
\usepackage{amsmath,amsfonts}
\usepackage{empheq}
\usepackage{array}
\usepackage{textcomp}
\usepackage{stfloats}
\usepackage{url}
\usepackage{verbatim}
\def\BibTeX{{\rm B\kern-.05em{\sc i\kern-.025em b}\kern-.08em
    T\kern-.1667em\lower.7ex\hbox{E}\kern-.125emX}}
\usepackage{balance}
\usepackage{diagbox}
\usepackage{algorithm}
\usepackage{pdfrender}
\usepackage{xcolor}
\usepackage{glossaries}
\usepackage{makecell}
\usepackage[usestackEOL]{stackengine}
\usepackage{array, multirow, cellspace}
\usepackage{pifont}
\usepackage{lipsum, tabularx, ragged2e, multirow}
\usepackage[utf8]{inputenc}
\usepackage{enumitem}
\usepackage{wrapfig}
\usepackage{colortbl,hhline}
\usepackage{flushend}
\usepackage{balance}
\usepackage{booktabs}
\usepackage{cleveref}
\usepackage[normalem]{ulem}
\usepackage[switch, columnwise]{lineno}
\usepackage{blindtext}

\def\({\left(}
\def\){\right)}

\def\[{\left[}
\def\]{\right]}

\usepackage{glossaries}

\newtheorem{theorem}{Theorem}

\newtheorem{definition}{Definition}
\newtheorem{constraint}{Constraint}
\newtheorem{remark}{Remark}

\providecommand{\best}[1]{\textcolor{red}{#1}}
\providecommand{\snd}[1]{\textcolor{green!50!black}{#1}}
\def\delequal{\mathrel{\ensurestackMath{\stackon[1pt]{=}{\scriptscriptstyle \Delta}}}}
\newcommand{\dbold}{\boldsymbol{d}^{\text{gen}}}
\newcommand{\dboldscript}{\boldsymbol{d'}^{,\text{gen}}}
\newcommand{\dopt}{\boldsymbol{d}^{*,\text{gen}}}
\newcommand{\F}{F(\boldsymbol{d}^{\text{gen}})}
\newcommand{\Fscript}{F(\boldsymbol{d'}^{,\text{gen}})}
\newcommand{\epshat}{\widehat\epsilon}
\newcommand{\ehat}{\hat \epsilon} 
\newcommand{\vcfl}{{\small \textsf{VCFL}}}
\newcommand{\wdg}{{\small \textsf{WDG}}}
\newcommand{\wco}{{\small \textsf{WCO}}}
\newcommand{\wpr}{{\small \textsf{WPR}}}
\newcommand{\radg}{{\small \textsf{RaDG}}}
\newcommand{\madg}{{\small \textsf{MaDG}}}
\newcommand{\proposed}{{\small \textsf{CoCoGen+}}}

\newcommand{\esdnfirder}{\epsilon'_{n}(\dgen)}

\newcommand{\dgen}{d^{{\text{gen}}}_n} 
\newcommand{\dgenopt}{d^{{*,\text{gen}}}_n} 

\newcommand{\dgenscript}{d'^{{\text{gen}}}_n} 
\newcommand{\Dset}{\mathcal{D}^{\text{gen}}_n} 

\newcommand{\dloc}{d^{{\text{loc}}}_n} 

\newcommand{\esilontotaldgen} {\epsilon({\boldsymbol{d}^\text{\text{gen}}})} 

\newcommand{\esilontotaldgenscript} {\epsilon({\boldsymbol{d'}^\text{,\text{gen}}})} 

\newcommand{\esilondngen}{\epsilon_{n}(d^\text{\text{gen}}_{n})} 
\newcommand{\CommentEq}[1]{\Comment{\textnormal{using Eq.~\eqref{#1}}}}
\newcommand{\CommentEqs}[2]{\Comment{\textnormal{using Eqs.~\eqref{#1} and \eqref{#2}}}}
\newcommand{\esilondngenscript}{\epsilon_{n}(d'^{,\text{gen}}_{n}} 

\newcommand{\esilondgenother}{\epsilon_{n'}(d^\text{\text{gen}}_{n'})} 

\newcommand{\dgenmin}{d^{{\text{gen}}}_{\text{min}}}
\newcommand{\dgenmax}{d^{{\text{gen}}}_{\text{max}}}
\newcommand{\potfunc}{F(\boldsymbol{d}^{\text{gen}})}
\newcommand{\dmix}{d_n^\text{\text{mix}}}

\newcommand{\xmark}{\ding{55}}%
\newcommand*\colourcheck[1]{%
  \expandafter\newcommand\csname #1check\endcsname{\textcolor{#1}{\ding{52}}}%
}

\usepackage{soul}
\crefname{equation}{Eq.}{Eqs.}

\newtcolorbox{boxIRQ}{
    colback = sub, 
    colframe = mainrq, 
    boxrule = 0pt, 
    toprule = 5pt,
    left=2pt,
    right=2pt,
    after skip=4pt,
    before skip=3pt,
    bottom=0pt,
    top=0pt
}
\definecolor{mainrq}{HTML}{3939c6}
\definecolor{main}{HTML}{A9A9A9}    
\definecolor{sub}{HTML}{f8f8ff}     
\definecolor{maroon}{cmyk}{0,0.87,0.68,0.32}
\definecolor{MyTeal}{RGB}{0,80,80}
\definecolor{MyPurple}{RGB}{125, 45, 175} 
\definecolor{MyOrange}{RGB}{160,75,0}
\definecolor{MyCrimson}{RGB}{170,35,55}

\newcolumntype{P}[1]{>{\centering\arraybackslash}p{#1}}

\makeatletter
\xpatchcmd{\proof}{\topsep6\p@\@plus6\p@\relax}{}{}{}
\makeatother

\colourcheck{blue}
\colourcheck{green}
\colourcheck{red}

\def\BibTeX{{\rm B\kern-.05em{\sc i\kern-.025em b}\kern-.08em
    T\kern-.1667em\lower.7ex\hbox{E}\kern-.125emX}}
\begin{document}

\title{{Cooperate to Compete: Strategic Data Generation and Incentivization Framework for Coopetitive Cross-Silo Federated Learning}}


\author{Thanh Linh Nguyen,~\IEEEmembership{Graduate Student Member,~IEEE}, Nguyen~Van~Huynh,~\IEEEmembership{Member,~IEEE}, and \\Quoc-Viet Pham,~\IEEEmembership{Senior Member,~IEEE}
\thanks{Thanh Linh Nguyen and Quoc-Viet Pham (corresponding author) are with the School of Computer Science and Statistics, Trinity College Dublin, Dublin 2, D02 PN40, Ireland (e-mail: \{tnguyen3, viet.pham\}@tcd.ie).}
\thanks{Nguyen Van Huynh is with the School of Computer Science and Informatics, University of Liverpool, Liverpool, L69 3DR, UK (e-mail: huynh.nguyen@liverpool.ac.uk).}
\thanks{Part of this article was presented at IEEE GLOBECOM 2025 \cite{nguyen2025coopetitive}.}
}

\maketitle


\begin{abstract}
In data-sensitive domains such as healthcare and finance, cross-silo federated learning (CFL) allows organizations to collaboratively train AI models without sharing raw data. However, many practical CFL deployments are inherently coopetitive, in which organizations cooperate during model training while competing in downstream markets. In such settings, training contributions, including data volume, quality, and diversity, can improve the global model yet inadvertently strengthen rivals, eroding one's competitive advantage. {This dilemma is amplified by non-independent and non-identically distributed data, which leads to asymmetric learning gains and undermines sustained participation}. While existing competition-aware CFL and incentive-design approaches reward organizations based on marginal training contributions, they fail to account for the implicit costs of strengthening competitors. In this paper, we introduce {\proposed}, a coopetition-compatible data generation and incentivization framework that jointly models statistical data heterogeneity and inter-organizational competition while endogenizing GenAI-based synthetic data generation as a strategic decision in CFL. Specifically, {\proposed} formulates each training round as a weighted potential game, {where organizations strategically decide how much synthetic data to generate by balancing learning performance gains against computational costs and utility losses caused by competition}. {We establish the existence and uniqueness of the continuous Nash equilibrium under the adopted model and develop a provably convergent scalar bisection algorithm. We further analyze how bilateral payoff redistribution modifies residual competition and data-generation incentives, and empirically evaluate the social welfare and individual utilities of the resulting implementable allocations}. Experiments on increasingly challenging learning tasks using Fashion-MNIST, CIFAR-10, and CIFAR-100 validate the feasibility of {\proposed}. The results show how heterogeneity, competition intensity, and incentives shape organizational strategies and social welfare, while {{\proposed} achieves greater social welfare improvements than competition-aware baselines.} 
\end{abstract}




\section{Introduction}
\label{sec:intro}
Modern artificial intelligence (AI) services and products increasingly depend on access to large, diverse, and high-quality data~\cite{alabdulmohsin2022revisiting}. However, in highly regulated domains, such as healthcare and finance, relevant data is often scarce and siloed across multiple organizations (e.g., hospitals or banks) and protected by strict privacy regulations, such as the General Data Protection Regulation (GDPR)\cite{regulation2016regulation}. These constraints are more pronounced in the era of generative AI (GenAI)\cite{openllm}, where training robust AI models requires massive and heterogeneous datasets, while centralized data collection is both costly and becoming harder to sustain due to legal and platform restrictions on large-scale data acquisition~\cite{oecd2025ipai}. In light of these constraints, cross-silo federated learning (CFL) has emerged as a promising paradigm, enabling organizations to collaboratively train a global model while keeping private data local, exchanging only model updates through a central server, preserving data privacy and regulatory compliance~\cite{hahn2025diffusion, kairouz2021advances}.

\begin{figure}[t!]
	\centering
	\includegraphics[width=0.95\linewidth]{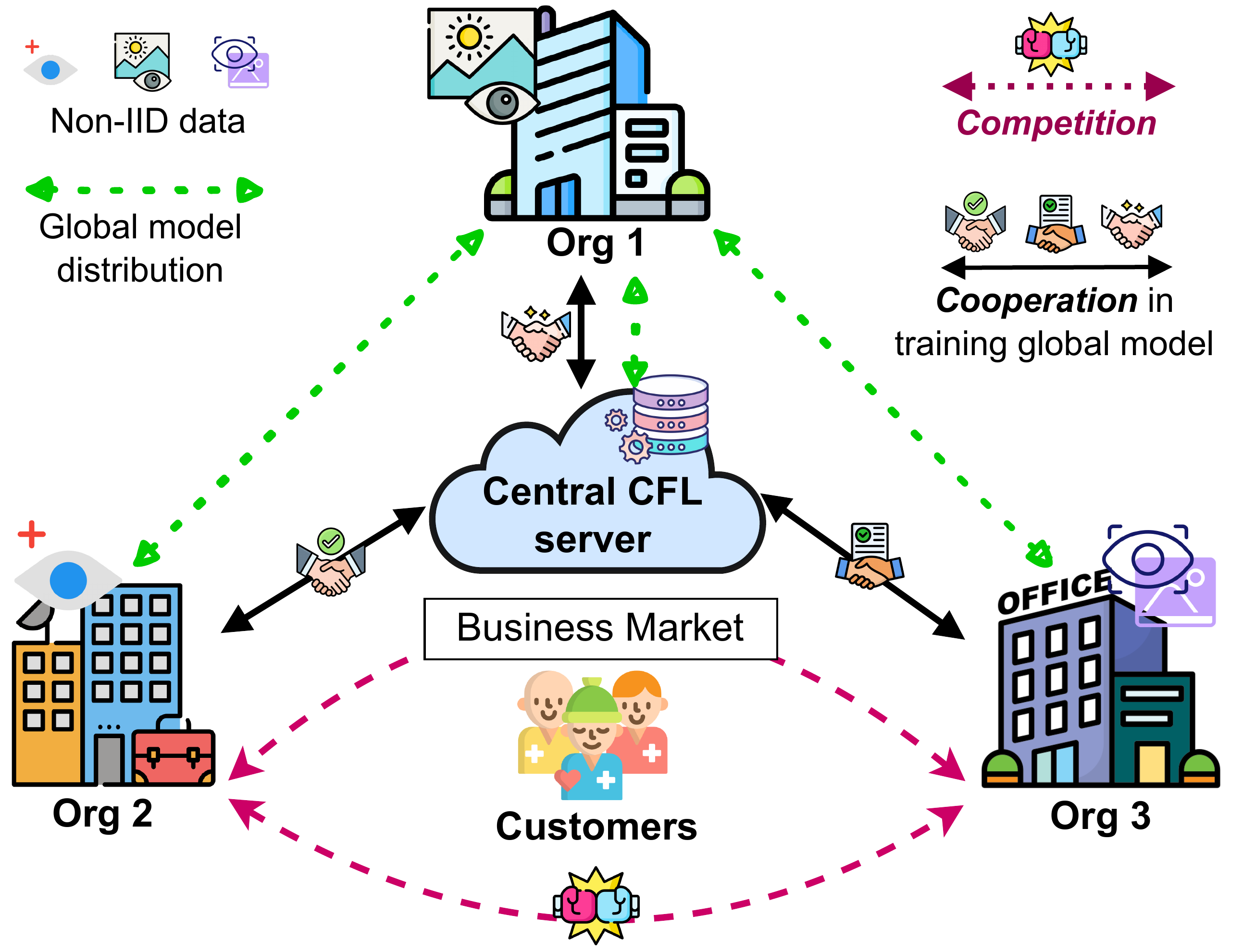}
	\caption{Coopetitive cross-silo federated learning setting with heterogeneous organizations. \textit{{Organizations cooperate by exchanging model updates to build a shared global model, while competing in downstream markets for market share and revenue. Each organization holds non-IID local data.}}}
    \label{fig:illustration_of_coopetition}
\end{figure} 

Despite its potential, CFL faces significant challenges in real deployments. {In particular, organizations are rational, benefit-driven, and often reluctant to fully contribute their proprietary data, computing, and communication resources, primarily due to the phenomenon of \textit{coopetition}, wherein organizations act as both \textit{collaborators} in model training (i.e., cooperation) and \textit{competitors} in leveraging the shared trained global model in downstream markets (i.e., competition)~\cite{huang2023duopoly, pmlreinav25a}}. This dual role fosters self-interested behavior, as organizations may fear exposing valuable information that may strengthen their competitors in the business market, alongside concerns about potential privacy leakage and the operational burden of iterative local training and communication.
{This challenge is compounded by the existence of non-IID data\footnote{In this work, non-IID data and data heterogeneity are used interchangeably.} across organizations, which can degrade global model generalization and cause uneven benefits across participating organizations~\cite{kairouz2021advances}}. Under non-IID, some organizations may gain little from cooperation relative to training alone, further increasing dropout risks, thereby weakening the sustainability of the CFL system. For example, Fig.~\ref{fig:illustration_of_coopetition} illustrates a scenario involving three geographically distributed private and for-profit healthcare organizations, each possessing heterogeneous datasets (e.g., retinal images from different patient populations and age groups). These organizations seek to cooperatively develop an AI-based automated diabetic retinopathy detection system to address the limitations of their individual datasets, such as insufficient data volume and diversity. However, as they serve overlapping customer bases and try to offer the most accurate and reliable AI solutions, they are also direct competitors. Consequently, concerns over exposing proprietary data, losing competitive advantage, and incurring a significant operational overhead present significant barriers to data sharing and global model development. Therefore, a key challenge in CFL is to quantify how coopetition and non-IID data jointly shape organizational incentives and to develop solutions that promote efficient cooperation through \textit{social welfare} optimization~\cite{tang2021incentive}. Social welfare is defined as the sum of the utilities of all participating organizations~\cite{murhekar2023incentives}, providing a system-level objective for balancing learning benefits against computational costs and competitive losses.

{Recent studies have shown that GenAI-based data augmentation can mitigate non-IID by generating synthetic data to reduce inter-organizational distribution gaps~\cite{10398474}. In competitive CFL, where data sharing is restricted and distributional mismatch across organizations is pronounced, GenAI can provide a scalable method to close these data distribution gaps without direct data exchange. However, GenAI-based synthetic data generation is computationally intensive and costly, particularly when performed on the organization side~\cite{openllm,zhang2025gpt, anonymous2026bridging}}. This introduces a consequential strategic decision: "\textit{how much synthetic data to generate on the local organization}". {Organizations need to balance the utility gains from improved learning performance against the costs of cooperative data generation and local training, as well as the competitive losses of enabling competitors to benefit from a more accurate shared global model. Despite its significance, existing CFL literature and incentive mechanisms have not yet provided a unified framework that endogenizes strategic GenAI-based data generation. Moreover, the coupled effects of data heterogeneity and inter-organizational competition on utilities and system-wide welfare remain underexplored.}

Motivated by these challenges and discussions, we aim to address the following research questions (RQs) in this work:

\noindent \textbf{\#RQ1:}\label{text:rq1} \textit{How does the interplay between non-IID data and economic inter-competition influence organizational strategic behaviors, incentive mechanism design, and the overall social welfare in a CFL system?}

\noindent \textbf{\#RQ2:} \label{rq:2}\textit{Given coupled effects of non-IID data and inter-organizational competition, how can we develop a GenAI-based data generation strategy for each organization to {{improve}} system-wide social welfare?}

{\noindent \textbf{\#RQ3:} \label{rq:3}\textit{How can we design a redistribution-based incentive mechanism that internalizes competitive externalities, promotes efficient data generation, and {{improves}} social welfare, while remaining budget-balanced under non-IID CFL?}}


To answer these RQs, we propose a {Co}opetition-{Co}mpatible Data {Gen}eration and Incentivization Framework for CFL, termed {\proposed}. {\proposed} jointly models inter-organizational competition and non-IID through gains based on learning performance (i.e., training loss) and payoff redistribution incentives. {It treats GenAI-based synthetic data generation as a strategic training decision, allowing organizations to optimize their strategies under competition and non-IID data conditions}. We show that the resulting per-round interaction can be formulated as a weighted potential game \cite{la2016potential}, which enables equilibrium characterization of organizations’ synthetic data generation strategies. {Building on this formulation, we establish the existence and uniqueness of the continuous Nash equilibrium and develop a provably convergent scalar bisection algorithm for its computation. We then obtain implementable integer allocations through post-processing and verify individual rationality}. {\proposed} incorporates a bilateral payoff-redistribution mechanism that uses differences in modeled local losses as proxies for relative contributions and directs compensation toward organizations with lower local losses, thereby promoting stable participation under practical constraints (e.g., budget balance). 

{The main contributions of this paper are as follows:\label{contributions} 
\begin{itemize}   
\renewcommand{\labelitemi}{\textcolor{black}{$\blacktriangleright$}}
    \item {We develop a unified CFL framework that jointly models data heterogeneity and inter-organizational competition while treating GenAI-based synthetic data generation as a strategic organizational decision. The framework captures the tradeoff between cooperative learning benefits, generation and training costs, and competitive losses.}
    
    \item {Building upon the bilateral payoff-redistribution concept introduced in~\cite{nguyen2025coopetitive}, we characterize organizations' data generation incentives through a residual competition coefficient $\theta_n$, explicitly establish budget balance under symmetric competition, and analyze how the compensation rate $\xi$ mitigates residual competitive losses associated with strengthening market rivals and shapes equilibrium data generation.}
    
    \item We establish that the strategic interaction within each training round is a weighted potential game, prove the existence and uniqueness of the continuous NE, and develop a provably convergent scalar bisection algorithm for its computation.

    \item {We provide empirical insights into the joint effects of competition, data heterogeneity, and payoff redistribution across three learning tasks. Particularly, stronger competition reduces synthetic-data generation and social welfare, while greater data heterogeneity increases generation demand but lowers social welfare. Moreover, increasing the compensation rate encourages data generation and improves social welfare, with benefits varying across learning tasks.}
\end{itemize}

{{While our preliminary conference version~\cite{nguyen2025coopetitive} introduced the bilateral payoff redistribution and a fixed-point heuristic on a single dataset, {\proposed} substantially extends and refines~\cite{nguyen2025coopetitive} in four ways. First, we clarify the local-loss-based contribution measure, reformulate the bilateral competition (i.e., Eq.~\eqref{eq:coopetition_loss_1}) and payoff-redistribution (i.e., Eq.~\eqref{eq:pay_off_redistribution_1}) functions using consistent local-loss gaps, explicitly establish budget balance under symmetric competition, and derive an exact weighted-potential representation (see Theorem~\ref{def:weighted_pg}) that separates nonlinear global loss improvements from residual competition (i.e., $\theta_n$, which is introduced in Section~\ref{subsec:ne_analysis}). Second, we establish the existence and uniqueness of the continuous Nash equilibrium and derive a scalar characterization that enables computation by provably convergent bisection, followed by integer post-processing and individual rationality verification (see Theorem~\ref{def:theorem_2} and Algorithm~\ref{alg:bisec}). Third, we analyze how the compensation rate $\xi$ affects residual competition, data generation incentives, and individual utilities, supported by new baselines. Finally, we expand validation to three datasets and model architectures.}}
\section{Related Work}
\label{sec:related_work}

\begin{table}[t!]
    \caption{Comparison between prior research \& \proposed.}
    \label{tab:comparison}
    \centering
    \begin{tabular}{|c|c|c|c|c|c|c|c|}
    \hline
    \rowcolor{gray!20}
\textbf{Work} & \textbf{F1} & \textbf{F2}&\textbf{F3} &\textbf{F4} & \textbf{F5}&\textbf{F6} &\textbf{F7} \\
\hline 
\hline 
        \cite{huang2023duopoly} & \xmark & \xmark &  \redcheck & \xmark & \redcheck & \xmark&  \xmark \\
    \cite{10530187} &\redcheck   & \xmark & \redcheck   & \xmark & \redcheck &\redcheck  &  \xmark\\
        \cite{wu2022mars, tan2024fedcompetitors, chen2024free} &\xmark  &  \xmark& \redcheck & \xmark &  \redcheck & \xmark & \xmark \\
    \cite{tang2021incentive, tang2024blockchain, mao2024game} & \redcheck & \xmark & \xmark & \xmark & \redcheck   & \redcheck  & \xmark \\
        \cite{10740334} & \redcheck & \xmark & \xmark & \xmark & \redcheck & \xmark & \xmark \\
    \cite{li2024filling} & \xmark &  \redcheck & \xmark &  \xmark & \xmark & \xmark & \xmark \\
    
         \cite{10680417} & \redcheck &  \redcheck & \xmark &  \xmark & \xmark & \xmark & \xmark \\
         
    \cite{huang2024imfl}  & \redcheck & \redcheck &  \xmark& \xmark & \xmark & \redcheck & \xmark \\ 
    
        \cite{yuan2023tradefl} & \xmark & \xmark &  \redcheck& \xmark & \redcheck & \redcheck & \redcheck \\ 

        \cite{murhekar2025you, murhekar2023incentives} & \redcheck & \xmark & \xmark &  \xmark & \xmark & \redcheck & \redcheck \\ 
        \hline

        \hline 
    \proposed & \redcheck & \redcheck & \redcheck & \redcheck & \redcheck & \redcheck & \redcheck \\\hline

    \end{tabular}

\begin{minipage}{0.95\linewidth} ~\\
\footnotesize	
     \textit{\textbf{F1}: Coopetition-aware incentive; \textbf{F2}: GenAI-based data generation; \textbf{F3}: Coopetition; \textbf{F4}: Joint non-IID data and coopetition modeling; \textbf{F5}: Cross-silo settings; {\textbf{F6}}:  Social welfare}; \textit{\textbf{F7}: Payoff redistribution protocol}.
\end{minipage}
\end{table}

\subsection{Competition-aware Cross-silo Federated Learning}\label{rel-work:competition}

\noindent Understanding inter-organizational competition in CFL has recently emerged as a critical research direction. {This competitive dynamic significantly influences organizations’ incentives to participate, the resources they contribute, and the value they can exploit from a shared global model, as well as from other organizations.} 

One line of research focuses on market-outcome modeling, where learning performance is linked to downstream economic outcomes (e.g., revenue, customer retention, or market share). For instance, Huang et al. \cite{huang2023duopoly} studied a duopoly setting and showed that competition, data contribution, and privacy costs jointly affect global model performance and organizational profits. Specifically, by formulating interactions among organizations and customers as a three-stage Stackelberg game, they showed that high privacy costs coupled with strong competition reduce profits and deteriorate model performance. Extending this direction, Huang et al. \cite{10530187} analyzed an oligopoly competition model and proposed a model-differentiation mechanism that allows organizations to further personalize the shared global model for downstream services. Their mechanism can improve learning model accuracy, organizational revenues, and social welfare. {Similarly, Wu et al. \cite{wu2022mars} developed a decision support framework among market competitors that models market-share dynamics based on customer loyalty, switching behaviors, and potential market growth. By introducing the concepts of market stability and friendliness, they established tight lower bounds on necessary performance gains, ensuring that model improvements are distributed such that no organization's market share decreases beyond a negotiated threshold.}


A second line of work addresses competition-aware collaboration formation. For example, Tan et al. \cite{tan2024fedcompetitors} leveraged graph-theoretic constructs and balance theory to quantify competitive relationships and proposed a collaborator-selection algorithm that avoids conflicts of interest while improving social welfare. 
{Chen et al. \cite{chen2024free} further proposed an optimal collaboration formation strategy, specifically designed to address the challenges of self-interest and business competition among organizations. The framework introduces two core principles of ensuring an organization only benefits from the system if they contribute to it, thereby eliminating free-rider concerns, where organizations will benefit from others' contributions without truthfully contributing to the training process, coalition formation, and preventing any contribution to competitors or their supporters to avoid conflicts of interest.}

\textit{While these studies and others in the literature establish that competition affects the CFL dynamics (e.g., social welfare, revenues, or choosing collaborators), they do not explicitly or theoretically quantify how non-IID drives learning-performance disparities that feed back into economic incentives, nor analyze the combined effects of non-IID and competition on organizational decision-making and social welfare~\cite{huang2023duopoly, 10530187, wu2022mars, tan2024fedcompetitors, chen2024free}. Furthermore, these works typically assume voluntary participation, overlooking the practical necessity of compensating organizations for their contributions, such as data, computation, and communication resources. This limitation is particularly critical in CFL, where all organizations receive the same final global model, while simultaneously competing in downstream tasks. {In such settings, the lack of contribution-aligned incentive designs may discourage organizations from allocating sufficient resources to the training process due to free-rider concerns, especially under competitive pressure~\cite{huang2023duopoly, 10530187, wu2022mars}.}}

\subsection{Incentive mechanisms in Cross-silo Federated Learning}\label{rel-work:incentive}
\noindent Game-theoretic incentive design has become a prevalent and effective tool for coordinating and sustaining CFL among rational and self-interested organizations~\cite{11626510, tang2021incentive, 10740334, mao2024game}. {Existing studies mainly focus on mitigating the free-rider problem and aligning private utilities with system-wide objectives (e.g., fairness, social welfare) under practical constraints such as individual rationality and budget balance.}


For example, Tang et al. \cite{tang2021incentive} proposed a non-cooperative game formulation to derive the organization's optimal processing capacity decision during training, explicitly targeting free-riding and public-good aspects of CFL, while maximizing the social welfare. They further integrated the incentive mechanism into a lightweight and transaction-efficient blockchain-based system to ensure transparency and auditability, eliminating dependence on a central party for incentive coordination in \cite{tang2024blockchain, 11626510}. From a privacy perspective, Mao et al. \cite{mao2024game} developed a multi-stage game-theoretic incentive mechanism to analyze the behavior of privacy-aware organizations and their impact on CFL model performance, particularly in terms of convergence rate. The mechanism incentivizes organizations to optimally select noise levels for local model updates while enabling the central server to allocate corresponding rewards. It further ensures truthful reporting and aims to achieve social welfare maximization. Liu et al. \cite{10740334} explored a different methodology for incentive mechanism designs by utilizing a supermodular game for characterizing strategic interactions and deriving data size strategies among organizations. The use of strategic complementarities captures the collaborative yet self-interested nature of CFL, where an organization’s incentive to contribute increases with the contributions of others. This structure ensures the existence and convergence of a pure-strategy Nash equilibrium, supporting a stable and sustained collaboration framework. The equilibrium outcomes are used to design a fairness-aware aggregation weighting mechanism, which addresses performance bias arising from data heterogeneity.

\textit{However, existing incentive mechanisms largely focus on cooperation-centric CFL settings and do not explicitly model business competition among organizations. As a result, they do not capture the competitive externality that an organization may suffer when its contributions improve competitors’ downstream services, an effect that can materially change resource-allocation decisions (e.g., the amount of GenAI-generated synthetic data), participation stability, global model performance, and social welfare even when standard incentive constraints are satisfied.}

\subsection{GenAI-based Federated Learning}\label{rel-work:genai}
\noindent{Non-IID data remains a central challenge in CFL, often causing optimization instability and deviation from the global optimum \cite{kairouz2021advances}. It also provides the key link between the two previous strands of work. Particularly, in competitive CFL, organizations need incentives to invest in GenAI-based synthetic data generation as a costly yet effective means of reconstructing underrepresented local distributions, reducing distribution gaps, and improving generalization.

For instance, Li et al. \cite{li2024filling} proposed FIMI, a resource-aware data generation framework that determines synthetic data amounts jointly with bandwidth and power allocation to meet performance targets while minimizing energy consumption to fill the missing categories in local datasets for cross-device FL. Zhang et al. \cite{10680417} considered offloading data generation from energy-constrained devices to generative AI service providers. The system models three entities, including the buyer, devices, and providers, who act as sellers. To incentivize participation and ensure efficient resource allocation, a reverse auction–based mechanism is employed to coordinate transactions among entities. This integration of GenAI and economic incentives leads to improved convergence speed and higher model accuracy. From a contract-theoretic viewpoint, Huang et al. \cite{huang2024imfl} proposed a data quality–aware incentive mechanism that jointly considers data volume, data quality (including local and GenAI-generated data), and unit computation cost. The framework derives optimal contracts to minimize the cloud server’s cost under information asymmetry settings, where only the distributions of client data sizes and computation costs are known. {This proposed solution achieves higher accuracy, better social welfare, and lower server costs.}

\textit{However, these GenAI-based approaches are primarily tailored to cross-device settings with resource-constrained clients and lack applicability to cross-silo FL, where organizations are strategic competitors, face different cost structures (e.g., for pricing or generating data), and must strategically decide whether and how much data to generate, while anticipating the spillover, also known as competitive externality \cite{scitovsky1954two}, to competitors in business markets. Consequently, GenAI-based synthetic data generation is not only a technical consideration but also a strategic and economic one, which remains insufficiently addressed in existing work.}

\begin{table}[t]
\centering
{{
\caption{Key notations.}
\label{tab:key_notation}
\renewcommand{\arraystretch}{1.1}
\setlength{\tabcolsep}{4pt}
\resizebox{\columnwidth}{!}{%
\begin{tabular}{l|l}
\hline
\rowcolor[HTML]{EFEFEF} 
\textbf{Notation} & \textbf{Description} \\
\hline\hline
$\mathcal{N}, N$ & Organization set and corresponding size \\
$\dloc$, $\dgen$ & Local and generated data sizes \\
$\dbold, \dopt$ & Data generation profile and continuous NE \\
$\dgenmin, \dgenmax$ & Lower and upper data-generation bounds \\
$\epsilon_n, \epsilon$ & Local and global loss functions \\
$\gamma_{n,n'}, \bar{\gamma}$ & Pairwise and average competition intensities \\
$\psi_n, \phi_n$ & Revenue rates for global-loss reduction \& contribution gaps \\
$\xi$ & Payoff compensation rate \\
$r_n, R_n, P_n$ & Cooperation gain, signed competition loss, payoff transfer \\
$C_n, C_0$ & Computation cost and server fee \\
$\theta_n, c_n$ & Residual competition coefficient, marginal computation cost \\
$U_n, SW$ & Individual utility and social welfare \\
\hline
\end{tabular}%
}
}}
\end{table}

\section{{System model} and problem formulation}
\label{sec:system_model}
\subsection{Overview}
\label{subsec:overview}
{Leveraging the concepts of coopetition \cite{gnyawali2009co} and external economies \cite{scitovsky1954two} from economics, we define CFL coopetition.}

\begin{definition} [CFL Coopetition]
\label{def:coopetition}
\underline{\textbf{CFL coopetition}} is a strategic regime in which organizations simultaneously \textit{\textbf{cooperate}} to train a shared global model and \textit{\textbf{compete}} in downstream markets for economic returns. Its defining feature is that an organization’s contribution generates positive collective learning gains but also produces negative competitive externalities\footnote{When an organization has a sheer volume of local/synthetic data, it disproportionately shifts inter-organizational performance gaps and 
contribution weights in the global model. Under competition, rivals may benefit from such contributions without compensating the contributing organization.} by unintentionally improving rival organizations’ market position.
\end{definition}

As illustrated in Fig.~\ref{fig:system_model}, we consider a practical CFL architecture with multiple coopetitive organizations and a trustworthy central server, which can be either rented or jointly established for coordination. These organizations, whose utilities depend on the model performance improvements across all participants \cite{wu2022mars}, form a consortium for model training.

{We denote by $\mathcal{N} \delequal \{1,\dots,n,\dots,N\}$ the set of $N$ organizations, where each organization $n \in \mathcal{N}$ has a processing capability $f_n$, and a dataset consisting of original local data and/or GenAI-generated synthetic data}. Organizations collaboratively participate in training a global model over a set of training rounds $\mathcal{T} \delequal \{1,..., t,..., T\}$ to improve performance on downstream tasks used for their services or products. Let $\mathcal{D}^{\text{loc}}_n$ and $\mathcal{D}^{\text{gen}}_n$ denote the original local data and synthetic data of organization $n$, respectively, with corresponding sizes $\dloc \delequal |\mathcal{D}^{{\text{loc}}}_n|$ and $\dgen \delequal |\mathcal{D}^{{\text{gen}}}_n|$. {In traditional CFL settings, organization $n$ trains its local model solely on dataset $\mathcal{D}^{{\text{loc}}}_n$. In contrast, in our work, each organization~$n$ adopts a mixed local dataset of $\mathcal{D}^{{\text{mix}}}_n = \mathcal{D}^{{\text{loc}}}_n \cup \mathcal{D}^{{\text{gen}}}_n$ and $d_n \delequal |\mathcal{D}^{{\text{mix}}}_n| = d_n^\text{\text{loc}} + d_n^{\text{gen}}$ for its local model update, where we apply the equal-generated data allocation strategy to each class}. Let $\boldsymbol{d}^{\text{gen}} \delequal \{\dgen: \forall{n} \in \mathcal{N}\}$ denote the GenAI-augmented data size strategy profile of all organizations in each round $t$, and let $\boldsymbol{d}^{\text{gen}}_{-n}$ be the data size strategy profile of all organizations excluding organization $n$ (i.e., $\boldsymbol{d}^{\text{gen}}_{-n} \delequal \{{d}^{\text{gen}}_{n'}:{\forall n', n \in \mathcal{N}, n'\neq n\})}$. {The key notations used throughout this paper are summarized in Table~\ref{tab:key_notation}.}


\begin{figure}[t!]
	\centering
	\includegraphics[width=\linewidth]{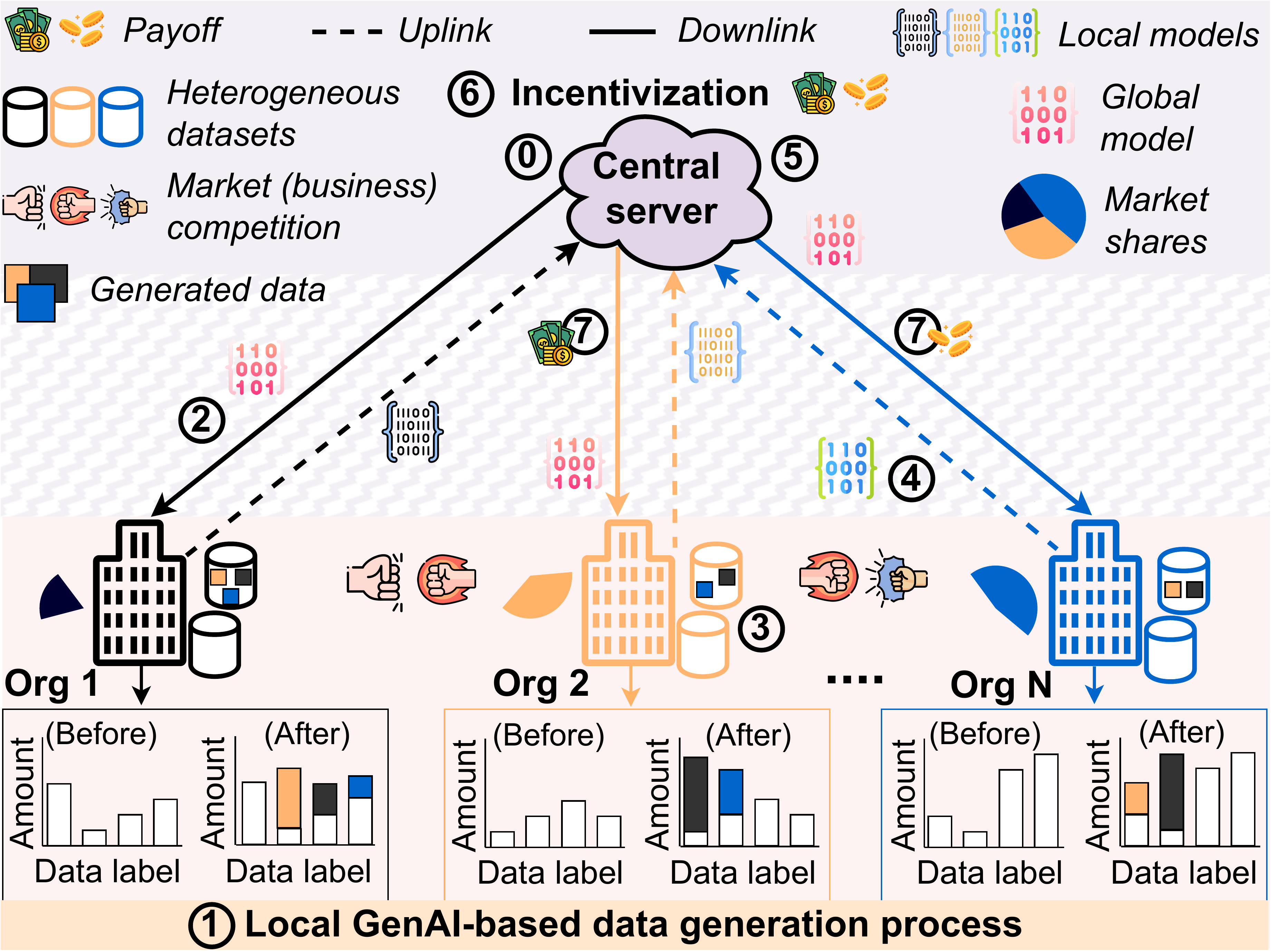}
	\caption{{\proposed} architecture with workflow.}
    \label{fig:system_model}
\end{figure} 

\subsection{GenAI-based Cross-silo Federated Learning Workflow}
\label{subsec:workflow}
CFL organizations collaboratively train the global model under the coordination of a central server using their private data.
Each organization owns a local model architecture with the same dimension as the global model architecture \cite{mcmahan2017communication}. \textit{The per-round CFL training process} is as follows (see Fig.~\ref{fig:system_model}). 
\begin{enumerate}
\setcounter{enumi}{-1}
    \item \textbf{Global model initialization \& Synthetic data generation strategy profile allocation}. 
    {We denote $\boldsymbol{w}^{t}$ as a set of global model parameters in the training round $t$}. Necessary information about the training process, such as the number of global model parameters and the learning network architecture, is public to all organizations within the system. Besides, based on gathered information (e.g., high-level information about data distribution, computational capabilities), which is approved by organizations for learning tasks, the central server can compute data generation strategies (as elaborated in Sec.~\ref{sec:proposed_mechanism}) and send this guidance to organizations.

    \item \textbf{Local GenAI-based data generation.} After receiving guidance from the central server, each organization $n$ puts efforts into augmenting its local dataset by generating $\dgen$ (e.g., using a variational transposed convolutional neural network \cite{anonymous2026bridging}), considering their local data, utilities, and competition intensity with other organizations.

    \item {\textbf{Global model downloading}. In round $t$, each organization $n$ downloads the global model parameter set $\boldsymbol{w}^{t-1}$.}

    \item  \textbf{Local model updating}. Then, each organization $n$ selects local data size ${d}_n$ including $\dloc$ and $\dgen$, and utilizes local compute resources $f_n$ to train $\boldsymbol{w}^{t-1}$ over one local update, performing stochastic gradient descent to obtain a local model update $w_n^t$. The local empirical loss of organization $n$ in the training round $t$ is defined by:
\begin{align}
\label{eq:local_loss}
L^{t}_n(\boldsymbol{w}^{t-1}) = \frac{1}{d_n} \sum\nolimits_{i=1}^{d_n} \ell(\boldsymbol{w}^{t-1}; x_i),
\end{align}
where $x_i$ is a data sample in the mixed dataset $\mathcal{D}^{{\text{mix}}}_n$ and $\ell(\boldsymbol{w}^{t-1}; x_i)$ is the classification loss function (e.g., cross-entropy) of training $\boldsymbol{w}^{t-1}$ on $x_i$. Based on \cite{ hestness2017deep, wang2020machine, li2025misfitting, li2024filling}, we use the power-law function to characterize  $n$'s local model performance (i.e., training loss) on the training data combination of private data and synthetic data, as expressed as follows: 
\begin{align}
\label{eq:local_loss_error}
\epsilon_{n} \approx \alpha(d^{{\text{loc}}}_{n} + d^{{\text{gen}}}_{n})^{-\beta} - \delta,
\end{align}
{where $\alpha > 0$, $\beta > 0$, and  $\delta \ge 0$ are experimentally determined hyperparameters\footnote{We can fit the performance curves using two ways, extrapolation or approximation,  depending on the learning tasks' complexity \cite{wang2020machine}.} (conducted in Sec.~\ref{sec:performance_evaluation}), with $\alpha$ representing a task-dependent scaling constant, $\beta$ controlling the steepness of the learning curve, and $\delta$ being a task-dependent offset}. {For a given learning task and data-heterogeneity setting, with each organization's generator held fixed,} this function shows that local model performance, thereby global model performance, depends on both local and synthetic data, underscoring the importance of optimizing the volume of synthetic data to achieve higher learning performance.

    \item \textbf{Local model uploading}. The organization $n$ uploads its model update ${w}_n^{t}$ to the server. 
    
    \item \textbf{Local model aggregation}. Finally, the server aggregates local model updates using an aggregation algorithm (e.g., FedAvg \cite{mcmahan2017communication}) to 
get a new global model $\boldsymbol{w}^{t}$. The global loss function is computed by:
\begin{align}
\label{eq:global_loss}
L^t(\boldsymbol{w}^{t}) = \sum_{n \in \mathcal{N}} \frac{d_n}{\sum^{N}_{i=1} d_{i}} L_{n}^{t}(\boldsymbol{w}^{t}).
\end{align}

    \item {\textbf{Payoff computation}. The central server computes payoff transfers based on differences in organizations’ modeled local losses, their competitive intensities, and the compensation rate. It then determines each organization’s net payment or compensation.}

    \item {\textbf{Reward distribution}. The central server implements the computed transfers by collecting payments from organizations required to compensate others and distributing compensation to the corresponding recipients.}
\end{enumerate}

\subsection{{Computational Cost}}\label{subsec:comp_costs}
Let $\eta_n$ and $\mu_n$ be the computation workload required for the organization $n$ to train and generate a single data sample, respectively.
The energy consumption in one training round of organization $n$ can be profiled by \cite{yang2020energy}:
\begin{align}
\label{eq:computation_overhead}
E_{n}^{\text{cmp}} = \kappa_n (\eta_n d_n + \mu_n d_{n}^{{\text{gen}}}) f_n^2,
\end{align}
where $\kappa_n$ is the effective capacitance coefficient of organization $n$'s computational chipset. The total cost for computation tasks of $n$ in each training round can be represented by:
\begin{align}
\label{eq:cost}
C_{n} = C_n^{\text{cmp}} E_{n}^{\text{cmp}},
\end{align}
where $C_n^{\text{cmp}}$ represents the computational operating cost per energy unit. Besides, each organization pays the central server the same fee $C_{0}$ for calculating the data generation strategy profile and aggregating the global model. {The size of model updates for each organization is fixed \cite{mcmahan2017communication}, and thereby the communication cost is neglected in this work \cite{sun2023understanding, yingexact}.} {In summary, {\proposed} accommodates differences in local-data availability, GenAI data generation workloads, and computational costs through the organization-specific parameters $\dloc$, $\mu_n$, $\kappa_n$, $f_n$, and $C_n^{\text{cmp}}$.}
\subsection{Organization Utility given Coopetition and Heterogeneity}
\label{subsec:utility_formulation}
{This subsection formalizes how each organization’s utility is shaped by the interplay of cooperative learning gains, competitive losses, computational costs, and payoff redistribution. We decompose the utility into four components and define each in turn.}

\paragraph{{Cooperation gain}}\label{para:cooperation_gain} Organization $n$ contributes a mixed dataset of size ${d}_{n}$ to achieve a high-accuracy global model (i.e., minimizing training loss). Following \cite{tang2021incentive}, we quantify each organization's cooperation gain through the reduction in global model performance (i.e., loss). We denote $\epsilon$ as the global model performance after one training round. The smaller the trained global performance, the better the fit of the global model to the training data. According to \cite{pandey2020crowdsourcing, yang2020energy}, {we establish the relationship between the global model performance $\epsilon$ under all organizations' data contribution strategy profile $\dbold=\left(\dgen, \boldsymbol{d}^\text{gen}_{-n}\right)$---organizations' strategy set $\boldsymbol{d}_{{-n}}^{{\text{gen}}}$, excluding $n$, keeps unchanged, and the local model performance $\epsilon_n$ (see Eq.~\eqref{eq:local_loss_error}), which is profiled as follows:}
\begin{align}
\label{eq:data_size_loss}
\esilontotaldgen = \text{exp}\left(\frac{\frac{1}{N}\sum\nolimits_{n \in \mathcal{N}}{\epsilon_{n}} - 1}{\varrho}\right),
\end{align}
where $\varrho > 0$ is a fine-tuning parameter based on learning tasks.
The revenue that organization $n$ {derives from collaborative learning depends on the improvement in the global model under the joint data-generation profile $\dbold$} is given by:
\begin{align}
\label{eq:cooperation_gain}
r_{n}(\dbold) = \psi_{n} \left[\epsilon_0 - \esilontotaldgen\right],
\end{align} 
where $\psi_{n} >0$ is the organization $n$'s revenue per unit reduction in global model loss~\cite{tang2021incentive, huang2023duopoly}, and {$\epsilon_0$ denotes the {loss} of the initial, untrained global model, evaluated on the test set}. {A reduction in global model loss therefore increases $r_n$, explicitly capturing the positive learning benefit of cooperation}. A higher $\psi_n$ means higher profitability that the organization $n$ gains from the global model-related services they can create. 

{\paragraph{Competition loss}\label{para:competition_loss} While cooperation improves the global model, it also creates a negative competitive externality, where rival organizations benefit from each other’s contributions without bearing the associated costs. To quantify this externality, we introduce the notion of competitive intensity.}
\begin{definition} [Competitive Intensity]
\label{def:competition_intensity}
CFL competitive intensity $\gamma_{n, {n'}} \in [0,1]$, for $n \neq n'$, quantifies the degree of market rivalry between organizations $n$ and $n'$. It captures the degree of strategic overlap in their global-model-derived products and the congruence of their target customer segments.
\end{definition}
{{We assume that competitive intensity is symmetric (i.e., $\gamma_{n,n'}=\gamma_{n',n}$), reflecting the mutual degree of competition between each pair of organizations.}} The local-loss gap is used as a proxy for organization $n$’s contribution relative to organization $n'$. The bilateral competition term is given by:
\begin{align}
\label{eq:coopetition_loss_1}
r_{n'} = \phi_{n'} \gamma_{n, n'} \left[\esilondgenother {-} \esilondngen\right],
\end{align}
where $\phi_{n'} \ge 0$ is the organization $n'$'s revenue per unit of contribution gap. The total competition loss that organization $n$ incurs aggregated over all rival organizations is given by:
\begin{align}
\label{eq:coopetition_loss_2}
R_{n}(\dbold) = \sum\nolimits_{n' \neq n} r_{n'}.
\end{align}
{ $R_{n}(\dbold)$ captures the economic consequences of relative learning performance among competing organizations. $R_n$ is a signed net competition term, where a positive value represents a competitive loss for organization $n$, and a negative value means a competitive benefit.
}
\paragraph{Payoff redistribution} To encourage cooperation among competing organizations, {\proposed} compensates higher-contributing organizations (i.e., lower training loss) at the expense of lower-contributing ones~\cite{nguyen2025coopetitive,yuan2023tradefl, falkinger2000simple}. The shared global model functions as a public good \cite{tang2021incentive}, without redistribution, free-riding incentives can erode participation. We therefore introduce a bilateral payoff transfer mechanism.
\begin{definition} [Payoff Redistribution]
\label{def:pay_off_redistribution}
Organization $n$ receives a payoff transfer from each competing organization $n'$, proportional to their modeled local-loss gap and scaled by their competitive intensity:
\begin{align}
\label{eq:pay_off_redistribution_1}
p_{n,n'} =  \xi \gamma_{n, {n'}}\left[\esilondgenother {-} \esilondngen\right],
\end{align}
where $\xi \ge 0$ is the payoff compensation rate for each contribution gap unit. Therefore, the total payoff of $n$ gained from its competitors $n'$, can be calculated by $P_{n}(\dbold) = \sum\nolimits_{n' \neq n} {p_{n, n'}}$.
\end{definition}

\paragraph{Discussion ($\xi\le\min_{n'}\phi_{n'}$)}\label{eq:assumption_payoff} The stability of the redistribution mechanism requires that the compensation rate $\xi$ not exceed the per-unit revenue $\phi_{n'}$ that any organization derives from its rival’s contribution.

\begin{remark}
\label{rem:1}
We impose $\xi\le\min_{n'}\phi_{n'}$\footnote{{We use $\xi$ and $\phi_{n'}$ to represent $P_{n}(\dbold)$ and $R_{n}(\dbold)$, respectively, for ease of presentation of the assumption because two equations are affected by these two parameters.}} to ensure that no organization $n'$ is required to compensate a rival $n$ at a rate exceeding the benefit it derives from that rival’s efforts in {\proposed}. This constraint preserves the economic rationality of the cooperation and is critical for participation stability under competition. It also accommodates regulated domains (e.g., healthcare, finance) where cooperation may be mandated despite potential economic losses.


\textbf{Example 1.} In the mergers and acquisitions sector, in Facebook’s acquisition of WhatsApp, the \$3B restricted stock unit allocation ($\xi$) was structured to vest over four years, capping payouts in proportion to the projected synergy revenue of $\phi_{n'}$, expected from WhatsApp’s integration. This ensures that the compensation for continued cooperation did not exceed the marginal utility Facebook anticipated gaining, aligning with our CFL assumption $\xi \leq \phi_{n'}$ \cite{campbell2014facebook}.

\textbf{Example 2.} In 5G standard-essential patent licensing, competing firms such as Ericsson, Nokia, and Qualcomm license their patented innovations, which are essential to the standard's functionality, to all manufacturers of end-user devices. Within this framework, the aggregate economic benefit generated by the universally adopted standard, monetized through licensing fees collected from all manufacturers, represents $\phi_{n'}$. This collected revenue is then distributed to patent-holding members via royalty payments ($\xi$). The financial viability and sustainability of the entire framework are predicated on the core assumption that $\xi \leq \phi_{n'}$, as the distributed royalties are by definition a fraction of the collected licensing fees \cite{stasik2020royalty}. 
\end{remark}

\begin{remark}[$\xi > \phi_{n'}$] When the compensation rate ($\xi$)  exceeds the benefit ($\phi_{n'}$), an organization must pay its competitor more than it gains, violating {\proposed}’s economic rationale. The dominant response in such a scenario is exit, blocking, or renegotiation. 


\textbf{Example 3.} In 2014, the Spanish government passed a new copyright law, mandating non-waivable payments ($\xi$) from news aggregators (e.g., Google News) to publishers for posting links or excerpts of news articles. However, Google News service's local benefit ($\phi_{n'}$) could not cover the compulsory fees, leading to their shutdown in Spain\cite{linktaxspain}.
\label{rem:contrast}
\end{remark}


\subsection{Problem Formulation}
\label{subsec:problem_formulation}
%
{Having defined four components of organizational utility, including cooperation gain $r_{n}(\dbold)$, computational cost $C_{n} + C_0$, competition loss $R_{n}(\dbold)$, and payoff redistribution $P_{n}(\dbold)$, we now combine them into a per-round utility function. The utility of organization $n$ when choosing a synthetic data generation strategy $d^{\text{gen}}_{n}$ is as follows:}
\begin{align}
\label{eq:client_utility_one_round}
{U}_{n}(\dbold) = & \; r_{n}(\dbold) + P_{n}(\dbold) - R_n (\dbold) -  C_{n} - C_{0}.
\end{align}
\textit{Social welfare} is defined as the sum of organizational utilities, specifically, 
\begin{align}
\label{eq:social_welfare}
SW(\boldsymbol{d}^{\text{gen}}) = \sum_{n \in \mathcal{N}} {U}_n(\dbold).
\end{align}
{\proposed} seeks a strategically stable synthetic data-generation profile {with improved system-wide social welfare}. The following two properties specify the requirements for individually rational participation and self-financing payoff redistribution.

\begin{constraint}[Individual Rationality--IR]
\label{def:ir}
Each organization $n$ participates in the CFL consortium only if its utility is non-negative, specifically, 
\begin{align}
\label{eq:ir}
{U}_{n}(\dbold) \ge 0, \forall{n \in \mathcal{N}}.
\end{align}
\end{constraint}

\begin{constraint}[Budget Balance]
\label{def:bb}
The payoff redistribution mechanism is self-financing, requiring no external subsidy \cite{falkinger2000simple}. Formally, aggregate payoff transfers sum to zero, i.e., 
\begin{align}
\label{eq:bb}
\sum_{n \in \mathcal{N}}\nolimits P_n(\dbold) = 0.
\end{align}
\end{constraint} 

{{
\begin{remark}
\label{remark:payoff_red}
Under symmetric competitive intensity, the bilateral transfers satisfy $p_{n,n'}=-p_{n',n}$. Consequently, summing the transfers over all organizations gives $\sum_{n \in \mathcal{N}}\nolimits P_n(\dbold) = 0$ for every strategy profile, ensuring budget balance without requiring external funding.
\end{remark}
}}
\section{A Potential Game-based proposed mechanism}\label{sec:proposed_mechanism}

In this section, we demonstrate that the interaction between coopetitive organizations is a weighted potential game. We then use it to derive GenAI-based data generation strategies.

\subsection{Game Formulation}\label{subsec:game_problem_formulation}
\textbf{\textit{Game $\mathcal{G}$:}} (Stage Game in the training round $t$)
\begin{itemize}[leftmargin=.3in]
\item \textit{Players:} competitive, heterogeneous, self-interested, rational, and strategic organizations $n \in \mathcal{N}$.
\item \textit{Strategies:} each organization $n$ decides how much data $d^{{\text{gen}}}_{n} \in \Dset \delequal[\dgenmin,\dgenmax]$ it should generate in the training round $t$, {{forming the joint strategy profile $\dbold \in \mathcal{D} \delequal \Pi_{n\in \mathcal{N}}{\Dset}$.}}
\item \textit{Objectives:} each organization $n$ aims to maximize its utility $U_n$ expressed in \eqref{eq:client_utility_one_round}.
\end{itemize}

\begin{definition}[Nash Equilibrium of $\mathcal{G}$]
\label{def:nash1}
A strategy profile $\boldsymbol{d}^{*,\text{gen}}$ is a pure-strategy Nash Equilibrium (NE) of the $\mathcal{G}$ if and only if no organization can improve its utility by deviating unilaterally \cite{nash1951non}, i.e., $U_n({\dgenopt}, {\boldsymbol{d}^{*,\text{gen}}_{\boldsymbol{-n}}}) \ge U_n({d^{{\text{gen}}}_{n}}, {\boldsymbol{d}^{*,\text{gen}}_{\boldsymbol{-n}}}).$
\end{definition}

\begin{algorithm}[t]
\small
\caption{Data-generation strategy  using scalar bisection}
\label{alg:bisec}
\begin{algorithmic}[1] 
\Require $\mathcal{N}$, $\psi_n$,$\phi_n$, $\xi$, $\gamma_{n, {n'}}$, $\epsilon_0$, $\dloc$, $\alpha$, $\beta$, $\delta$, $\varrho$, $\kappa_n$, $C^{\text{cmp}}_n$, $C_0$, $f_n$, $\eta_n$, $\mu_n$,  $\dgenmin$, $\dgenmax$, tolerance $\epsilon_{\text{tol}}$, maximum iterations $K_{\text{max}}$

\Ensure Approximate continuous NE $\widetilde{\boldsymbol d}^{\mathrm{gen}}$, integer allocation ${\boldsymbol d}^{\mathrm{gen}}$, convergence status, and IR status

\noindent $\textcolor{gray}{\rhd{\textit{ Parameter initialization:}}}$
\For {each $n = 1$ to $N$}
\State ${\theta_n} = \sum_{n' \neq n} \gamma_{n, {n'}} (\phi_{n'} - \xi) \CommentEq{eq:theta_and_c}$\label{alg:theta}
\State $c_n = \kappa_nC_n^{\text{cmp}} (\eta_n  +\mu_n) f_n^2 \CommentEq{eq:theta_and_c}$\label{alg:c_n}
\EndFor

\noindent $\textcolor{gray}{\rhd{\textit{ Internal initialization}}}$

\State $k \gets 0,\;$ $l \gets \epsilon(\dgenmax\mathbf{1}),\;$ $r \gets \epsilon(\dgenmin\mathbf{1}) \CommentEqs{eq:local_loss_error}{eq:data_size_loss}$

\If {$h(l)=0$} $r=l \CommentEq{eq:bis_equ}$ 
\ElsIf {$h(r)=0$} $l=r \CommentEq{eq:bis_equ}$
\EndIf

\noindent $\textcolor{gray}{\rhd{\textit{ Scalar bisection using Theorem~\ref{def:theorem_2}:}}}$

\While{$(r - l)/2 > \epsilon_{\text{tol}}$ and $k < K_{\text{max}}$}\label{alg:bisec_oper} 
\State $\ehat \gets l +  (r-l)/2$
\State $\boldsymbol{d}_{\text{tmp}}^{\text{gen}} \gets \dbold(\epshat) \CommentEq{eq:solution}$
\State $\epsilon \gets \epsilon(\boldsymbol{d}_{\text{tmp}}^{\text{gen}}) \CommentEqs{eq:local_loss_error}{eq:data_size_loss}$
\State $\hat{h} \gets \epsilon - \ehat \CommentEq{eq:bis_equ}$
\State $k \gets k+1$

\If {$\hat{h}=0$} $(l, r) \gets (\ehat, \ehat);$ \textbf{break}
\ElsIf {$h(\ehat)>0$} $(l, r) \gets (\ehat, r)$
\Else {\text{ }$(l, r) \gets (l, \ehat)$}
\EndIf
\EndWhile

\noindent $\textcolor{gray}{\rhd{\textit{ Convergence verification:}}}$
\If{$(r - l)/2 > \epsilon_{\text{tol}}$}
\State \Return non-convergence status
\EndIf

\State $\ehat \gets l+(r-l)/2;\;\;$ $\widetilde{\boldsymbol{d}}^{\mathrm{gen}} \gets \dbold(\ehat)\CommentEq{eq:solution}$

\noindent $\textcolor{gray}{\rhd{\textit{ Integer post-processing using Theorem~\ref{def:weighted_pg}:}}}$

\State  $\dbold \gets \widetilde{\boldsymbol{d}}^{\mathrm{gen}} $\label{alg:int_round}
\For {each $n = 1$ to $N$}
\State $\mathcal M_n \leftarrow \left\{ \left\lfloor\widetilde d_n^{\text{gen}}\right\rfloor, \left\lceil\widetilde d_n^{\text{gen}}\right\rceil \right\} \cap [\dgenmin,\dgenmax]$
\State $\dgen \leftarrow \arg\min_{v\in\mathcal M_n} F\!\left( v,{\boldsymbol d}_{-n}^{\text{gen}} \right) \CommentEq{eq:theorem1}$
\EndFor

\noindent $\textcolor{gray}{\rhd{\textit{ IR verification at the final integer allocation:}}}$
\State $\mathrm{IR} \leftarrow \left[ \min_{n\in\mathcal N} U_n\!\left(\dbold\right) \ge 0 \right] \CommentEq{eq:ir}$\label{alg:ir_ver} \label{alg:int_round_end}
\State \Return $(\widetilde{\boldsymbol{d}}^{\mathrm{gen}}, {\boldsymbol{d}}^{\mathrm{gen}}, \text{converged}, \mathrm{IR})$
\end{algorithmic}
\end{algorithm}
\setlength{\textfloatsep}{2pt}

\subsection{Nash Equilibrium Analysis}\label{subsec:ne_analysis}

Directly characterizing the NE of $\mathcal{G}$ is challenging because organizations' best responses are coupled through the global learning loss $\esilontotaldgen$. We first analyze the continuous relaxation of $\mathcal G$, where each organization's data generation decision can take any real value in $[\dgenmin, \dgenmax]$. {{Using the weighted-potential-game framework \cite{monderer1996potential}, Theorem~\ref{def:weighted_pg} establishes that the relaxed game admits a weighted potential function $\Phi = -F$ with positive weights, implying that every global minimizer of $F$ is an NE. Then, Theorem~\ref{def:theorem_2} further establishes that the NE is unique and provides a scalar characterization of this NE, enabling its computation by bisection. Throughout this analysis, we define
\begin{align}
\label{eq:theta_and_c}
    {\theta_n} = \sum_{n' \neq n} \gamma_{n, {n'}} (\phi_{n'} - \xi), \quad c_n = \kappa_nC_n^{\text{cmp}} (\eta_n  +\mu_n) f_n^2.
\end{align}
Here, $\theta_n \geq 0$ represents the residual competition coefficient after compensation given $\gamma_{n, {n'}} \geq 0$ and $\xi\le\min_{n'}\phi_{n'}$, and $c_n > 0$ is the marginal computational cost of generating and training on an additional data sample. Parameter estimation and sensitivity analysis of organizations' marginal incentives for synthetic-data generation are provided in Appendix~\ref{subsec:param_est}.
}}

{{
\begin{theorem}
\label{def:weighted_pg}
The game $\mathcal{G}$ is a weighted potential game with positive weights $\psi_n$ and potential function  $\Phi = -F$, where
\begin{align}
\label{eq:theorem1}
F(\dbold) = \esilontotaldgen - \sum_{n \in \mathcal{N}} \frac{\theta_n}{\psi_n} \epsilon_n(\dgen) + \sum_{n \in \mathcal{N}} \frac{c_n}{\psi_n} \dgen.
\end{align}
\end{theorem}

\noindent For any feasible unilateral deviation by organization $n$, let $\Delta U_n, \Delta F$, and $\Delta \Phi$ denote the corresponding changes in its utility and the two functions. As a result, $\Delta U_n = -\psi_n \Delta F = \psi_n \Delta \Phi$.}}
\begin{proof}
    The proof can be referred to Appendix~\ref{proof:theorem1}.
\end{proof}

Since $\psi_n >0$, every unilateral utility improvement corresponds to a decrease in $F$. Consequently, every global minimizer of $F$ over $\mathcal{D}$ is an NE. This motivates the continuous potential-minimization problem as follows:
\begin{subequations}
\label{eq:optim_func2}
{{
\begin{align}
\label{eq:optim_func3}
\min_{\dbold} \quad & \potfunc, \\  
\noindent \textrm{s.t.} \quad  & \dgenmin \leq \dgen \leq \dgenmax, \;\; \dgen\in \mathbb{R}_+, \forall n \in \mathcal{N}.
\end{align}}}
\end{subequations}
{The following theorem establishes that problem~\ref{eq:optim_func2} has a unique minimizer, characterizes the corresponding NE, and reduces its computation to a scalar equation.}
{{
\begin{theorem}
\label{def:theorem_2} 
The game $\mathcal{G}$ admits a unique NE $\dopt = \{\dgenopt: \forall n \in \mathcal{N}\}$, which is also the global minimizer of $\F$ over $\mathcal{D}$. For a scalar trial global loss $\ehat > 0$, we define the candidate strategy of organization $n$ as follows:
\begin{align}
\label{eq:solution}
    \dgen(\epshat) = \Pi_{[\dgenmin, \dgenmax]}\left[\left(\frac{\alpha\beta}{c_n}\left[\frac{\psi_n \epshat}{N\varrho} - \theta_n\right]\right)_{+}^{1/(\beta+1)} - \dloc\right],
\end{align}
where $[z]_{+} = \max\{z,0\}$ and $\Pi$ denotes projection onto the feasible data generation interval. The NE satisfies $\dopt = \dbold(\ehat^*)$, where $\ehat^*$ is the unique solution of 
\begin{align}
\label{eq:bis_equ}
    h(\ehat) \delequal \epsilon(\dbold(\ehat)) - \ehat = 0,
\end{align}
in the interval of $[l_0, r_0] = [\epsilon(\dgenmax\mathbf{1}), \epsilon(\dgenmin\mathbf{1})]$. The residual function $h$ is continuous and strictly decreasing, with $h(l_0) \geq 0$ and $h(r_0) \leq 0$. Therefore, bisection converges to $\ehat^*$, and the corresponding candidate strategy profiles converge to $\dopt$.

\label{theo:2}
\end{theorem}
}}
\begin{proof}
The proof can be referred to Appendix~\ref{proof:theorem2}. 
\end{proof}

{{Here, $F$ characterizes organizations' unilateral incentives, whereas $SW$ sums their utilities. Theorem~\ref{def:theorem_2} establishes the unique NE of the continuous game $\mathcal{G}$ and convergence of its computation within a given learning round, but does not generally imply $SW$ maximization. We evaluate the $SW$ of the resulting data generation strategies empirically in Section~\ref{sec:performance_evaluation}.}}

{{
Algorithm~\ref{alg:bisec} approximates the continuous NE $\boldsymbol{d}^{*,\text{gen}}$ and obtains a feasible integer allocation. IR in Eq.~\eqref{eq:ir} is checked at the final integer allocation, while budget balance in Eq.~\eqref{eq:bb} holds by construction through antisymmetric bilateral transfers. Parameter initialization, integer post-processing, and utility evaluation for IR verification each require at most $\mathcal{O}(N^2)$ operations. Each bisection iteration requires $\mathcal{O}(N)$ operations, giving an overall complexity of $\mathcal{O}(N^2+NK)$ after $K$ iterations. Comparative complexity analysis and runtime measurements are provided in Appendix~\ref{proof:runtime}.
}}

\section{Experiments}
\label{sec:performance_evaluation}

This section evaluates {\proposed} with three objectives. First, we assess how well the adopted scaling law captures the relationship between training-data volume and local learning loss, and examine how competition and data heterogeneity influence organizations' data-generation decisions and the resulting social welfare (\textbf{\#RQ1, \#RQ2}). Second, we investigate how payoff redistribution affects data generation incentives, IRs, and social welfare (\textbf{\#RQ1, \#RQ3}). Third, we compare the social welfare achieved by {\proposed} with baseline methods across varying competition intensities, data-heterogeneity levels, and learning-task complexities (\textbf{\#RQ1, \#RQ2, \#RQ3}).
\begin{table}[t!]
    \caption{{Parameter settings.}}
    \label{tab:parameters}
    \centering
    \resizebox{\columnwidth}{!}{%
    \begin{tabular}{|ll|ll|}
        \hline
        \rowcolor[HTML]{EFEFEF} 
        \textbf{Param}  & \textbf{Value} & \textbf{Param} & \textbf{Value}  \\ \hline \hline
        $\kappa_n$ & $10^{-28}$~\cite{tran2019federated} & $\phi_n$ & $\mathcal{U}(200,400)$  \\
        $\varrho$ & 6  & $\dloc$ & $\mathcal{U}(50,150)$ samples  \\
        $f_n$ & $\mathcal{U}(1.5, 2)$~GHz  & $\xi$  & $\{10,50,90,130,200\}$ \\ 
        $N$ & 10  & $\gamma_{n, n'}$ & $\{\mathcal{U}(0, 0.2), \mathcal{U}(0, 1),$ \\ 
            &     &                  & $\phantom{\{}\mathcal{U}(0.8, 1)\}\times2\times 10^{-3}$ \\ 
        $\mu_n$ & $3\times10^{7}$   & $\alpha_D$ & $\{0.1, 0.5, 0.9\}$ \\ 
        $d_{\text{min}}^{\text{gen}}$ & 0 samples & $d_{\text{max}}^{\text{gen}}$ & 24{,}000 samples \\
        $\eta_n$ & $3\times10^{6}$\cite{li2024filling} & $\psi_n$ & $\mathcal{U}(400,600)$ \\
        $C_n^{\text{cmp}}$ & 1 & $\epsilon_0$ & 2.5 \\
        $C_0$ & 0 & $\epsilon_{\text{tol}}$ & $10^{-12}$ \\
        \hline
    \end{tabular}%
    }
\end{table}

\subsection{Experiment Setup}
\label{subsec:exp_setup}

{We consider a CFL system with $N = 10$ organizations, consistent with the typically small number of participants in CFL~\cite{kairouz2021advances}. To represent different levels of inter-organizational competition, we set the competitive intensity to $\bar\gamma = \{0.1, 0.5, 0.9\}\times 2 \times 10^{-3}$, where each value is computed as the mean of competitive intensity across all organizations with individual intensities drawn from $\{\mathcal{U}(0, 0.2), \mathcal{U}(0, 1),\mathcal{U}(0.8, 1)\}\times2\times 10^{-3}$, respectively. These three cases correspond to low, moderate, and high competition.}


\textbf{Datasets.} Our experiments are conducted on three datasets, including Fashion-MNIST (FMNIST)~\cite{xiao2017fashion}, CIFAR-10 \cite{krizhevsky2009learning}, and CIFAR-100 \cite{krizhevsky2009learning}, with different learning difficulty levels being simple, medium, and complex, respectively \cite{qin2023fedapen}. To model data heterogeneity, we sample label proportions $p\sim\textit{Dir}{(\alpha_D)}$, where $\alpha_D = \{0.1, 0.5, 0.9\}$ is the Dirichlet distribution parameter to control the degree of label imbalance. A smaller $\alpha_D $ means a higher heterogeneity level.

\textbf{Implementation.} We conduct experiments over three datasets with Flower framework \cite{beutel2020flower} on 4 NVIDIA L40 GPUs with 48GB memory. For FMNIST and CIFAR-10, we use MobileNetV3-small \cite{howard2019searching} and MobileNetV2 \cite{sandler2018mobilenetv2}, respectively. For CIFAR-100, we use ResNet-34 \cite{he2016deep}. The batch size is 16, and the learning rate is 0.01. {Unless otherwise stated, key parameter settings are shown in Table~\ref{tab:parameters}. {Appendix~\ref{subsec:param_est} clarifies how the scaling-law-based learning parameters are empirically fitted and the economic scenario parameters are specified for our experiments.}} 

\begin{figure*}[ht!]
\captionsetup[subfigure]{skip=5pt} 
\centering
\begin{subfigure}[t]{.327\textwidth}
    \centering
    \includegraphics[width=\linewidth]{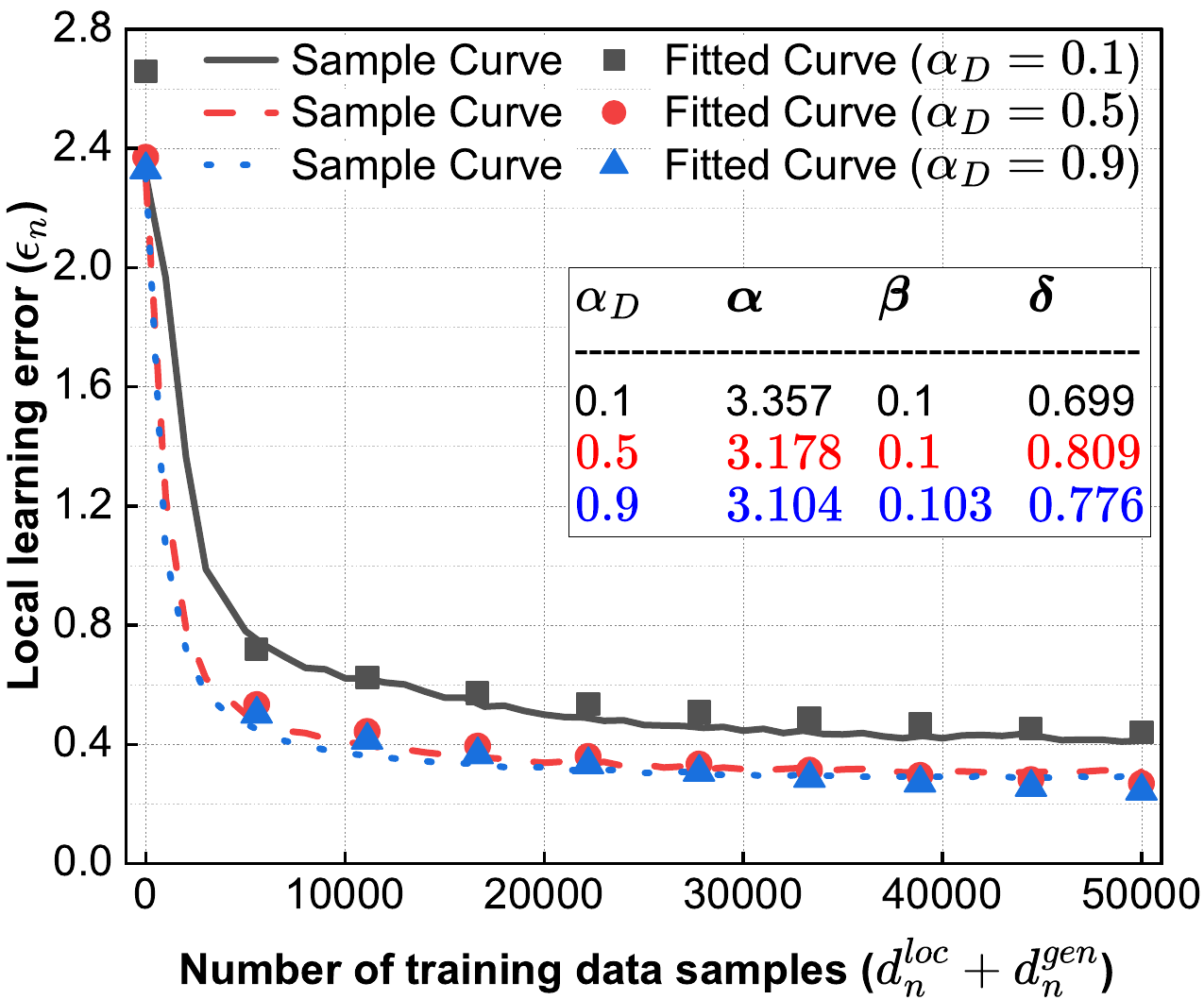}  
    \caption{FMNIST with MobilenetV3-small.}
    \label{fmnist:scaling_laws}
\end{subfigure} 
\begin{subfigure}[t]{.327\textwidth}
    \centering
    \includegraphics[width=\linewidth]{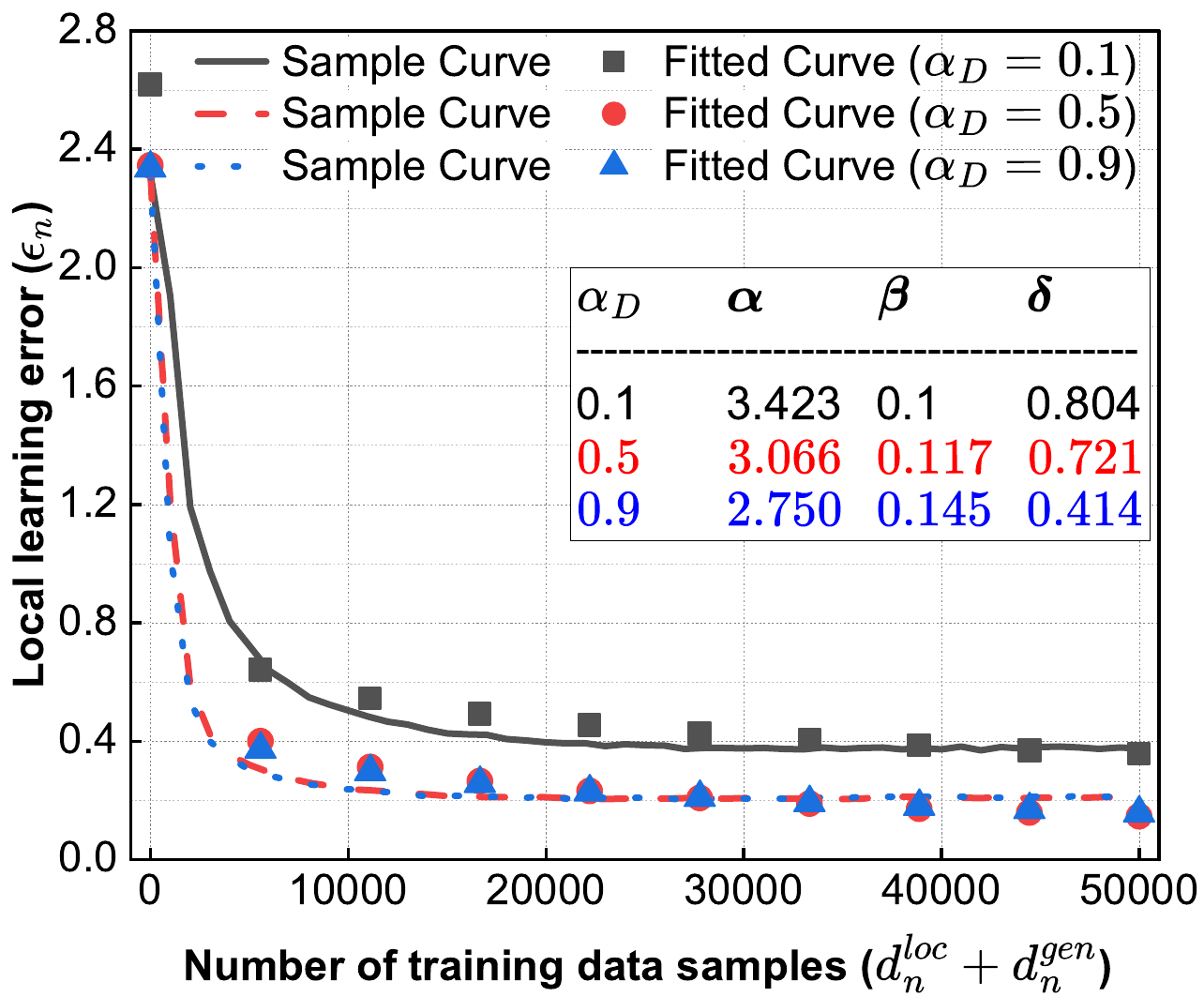}  
    \caption{CIFAR-10 with MobileNetV2.}
    \label{cifar10:scaling_laws}
\end{subfigure}
\begin{subfigure}[t]{.327\textwidth}
    \centering
    \includegraphics[width=\linewidth]{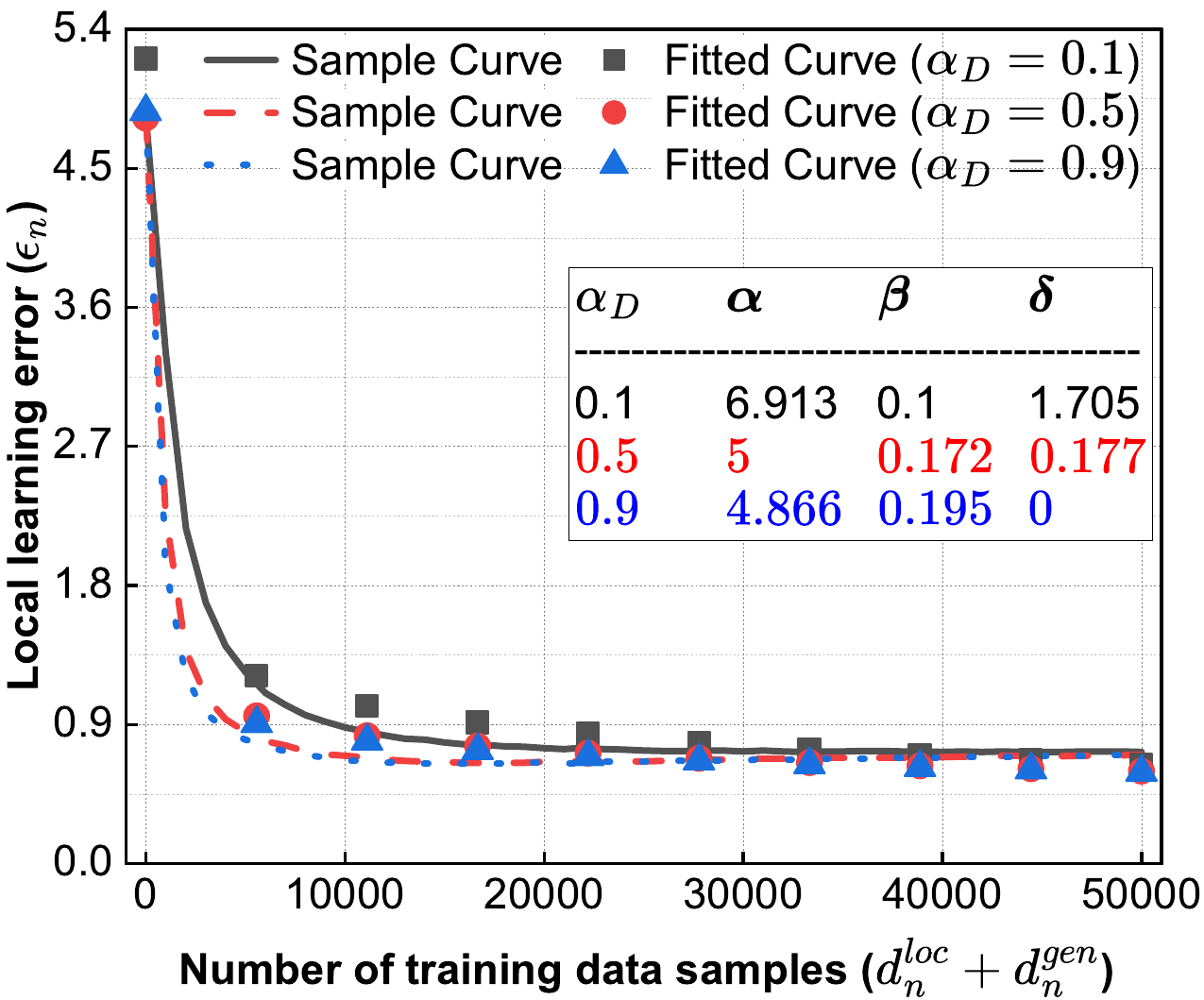}  
    \caption{CIFAR-100 with ResNet-34.}
    \label{cifar100:scaling_laws}
\end{subfigure} 
\caption{$\epsilon_n$ with respect to the number of local training samples using different datasets and deep neural network architectures.}
\label{fig:hyperparameter}
\end{figure*}
\textbf{Baselines.} Existing work offers limited applicable baselines for coopetition-aware data generation among organizations in heterogeneous CFL settings. We therefore adopt the following approaches as comparative baselines.
\begin{itemize}
    \item \textbf{Vanilla CFL (VCFL).} Traditional CFL without GenAI-based data generation and without competition among organizations (i.e., $\dgen = 0, \gamma_{n,n'} = 0, n' \neq n$).

    \item {{\textbf{Without Payoff Redistribution (\wpr).} Payoff redistribution is disabled (i.e., $\xi = 0$), while competition remains active. Organizations generate the NE amounts of the resulting game, obtained by Algorithm~\ref{alg:bisec}. Comparing {\wpr} with {\proposed} isolates the effect of payoff redistribution.}}
    
    \item  \textbf{Without Competition among Organizations (\wco).} There is no competition among organizations (i.e., $\gamma_{n,n'} = 0, n' \neq n$), {{so the payoff transfers in Eq.~\eqref{eq:pay_off_redistribution_1} also vanish. Data-generation allocations are computed using Algorithm~\ref{alg:bisec}  for the resulting competition-free game. Therefore, {\wco} serves as an idealized reference for assessing the social welfare gap associated with competition.}}
    
    \item  \textbf{Without Data Generation (\wdg).} In \wdg, GenAI-based data generation mechanisms are not applied, but the payoff redistribution-based incentive solution is still used (i.e., $\dgen = 0, \gamma_{n,n'} \neq 0, \forall n, n' \in \mathcal{N}, n' \neq n$).
    \item \textbf{Random Data Generation (\radg).} In this {\radg} scheme, organizations randomly generate the amount of GenAI-based augmented data for local training.
    \item \textbf{Maximum Data Generation (\madg).} In this {\madg} scheme, organizations exhaust their local resources to generate the maximum amount of GenAI-based augmented data for local training (i.e., $\dgen = \dgenmax, \forall n \in \mathcal{N}$).
    \item {{\textbf{Fixed-half.} Each organization generates half of the generation cap (i.e., $\dgen = \dgenmax/2 = 12{,}000$ data samples).}}
\end{itemize}
In our analysis, {\radg}, {\madg}, and Fixed-half use the same competitive intensities and compensation rate as {\proposed}. These baselines differ from {\proposed} in their data-generation allocation rules. Moreover, comparisons with {\vcfl} measure welfare gains relative to a no-generation, competition-free reference. Comparisons with {\wdg} assess the benefit of data generation under competition, while comparisons with {\wpr} assess the effect of payoff redistribution.

{ {
\textbf{Evaluation metrics.} Results are averaged over $30$ random seeds. We evaluate the following metrics:
\begin{itemize}
    \item \textbf{Social welfare and data generation:} Social welfare ${\text{SW}}_m$ of method $m$ is defined by Eq.~\eqref{eq:social_welfare} and the average number of generated data samples per organization is $\overline{d}^{\mathrm{gen}}=\frac{1}{N}\sum_{n\in\mathcal{N}} \dgen$.  
    \item \textbf{Welfare gain over \vcfl:} $\Delta\mathrm{SW}_m=\mathrm{SW}_m-\mathrm{SW}_{\mathrm{VCFL}}$ measures additional welfare achieved relative to {\vcfl}. While Fig.~\ref{fig:impacts_alphaD_gamma_on_3_datasets} shows social welfare, Table~\ref{tab:social_welfare_efficiency} reports social welfare gains over {\vcfl}
    \item \textbf{Normalized welfare gain:} $\omega_m=\Delta\mathrm{SW}_m/\Delta\mathrm{SW}_{\mathrm{WCO}}$ expresses welfare gain relative to the competition-free reference, where $\omega_m=0$ corresponds to {\vcfl} and $\omega_m=1$ to {\wco}. We denote $\omega(\xi)$ for {\proposed} with payoff compensation rate $\xi$.
    \item \textbf{Recovered welfare:} $(\mathrm{SW}_{\proposed}-\mathrm{SW}_{\mathrm{WPR}})/(\mathrm{SW}_{\mathrm{WCO}}-\mathrm{SW}_{\mathrm{WPR}})$ measures the fraction of welfare gap between {\wpr} and {\wco} recovered through payoff redistribution.
\end{itemize}

}}

\subsection{\proposed \text{ Feasibility}}
\label{subsec:feasibility}

\paragraph{{Scaling laws for CFL}} We demonstrate that the scaling law function (see Eq.~\eqref{eq:local_loss_error}) effectively captures the interplay of data heterogeneity, data quantity, comprising both local and newly generated data, and the local learning performance (i.e., loss) across different datasets and model architectures in CFL.  As shown in Fig.~\ref{fig:hyperparameter}, for each heterogeneity level $\alpha_D$, we empirically obtain a corresponding set of hyperparameters $\{\alpha, \beta, \delta\}$ that characterize the local learning curve on three datasets. As the volume of GenAI-generated data increases while the number of original local data samples per organization remains fixed, the local learning error decreases significantly across different $\alpha_D$ values. This reduction saturates once the generated data reaches a sufficient volume, at which point the loss stabilizes. Notably, our hyperparameterized scaling law yields learning curves that closely match the empirical trajectories, supporting the validity and feasibility of the proposed model. These results demonstrate that GenAI-based data augmentation not only improves local model performance but also enables \textit{predictable} and \textit{controllable} improvements in downstream tasks in federated settings. Consequently, our framework provides a principled basis for deciding when and to what extent to scale synthetic data generation under different heterogeneous settings.

\begin{figure}[t!]
\captionsetup[subfigure]{skip=5pt} 
\begin{subfigure}[t]{0.492\textwidth}
    \centering
    \includegraphics[width=\linewidth]{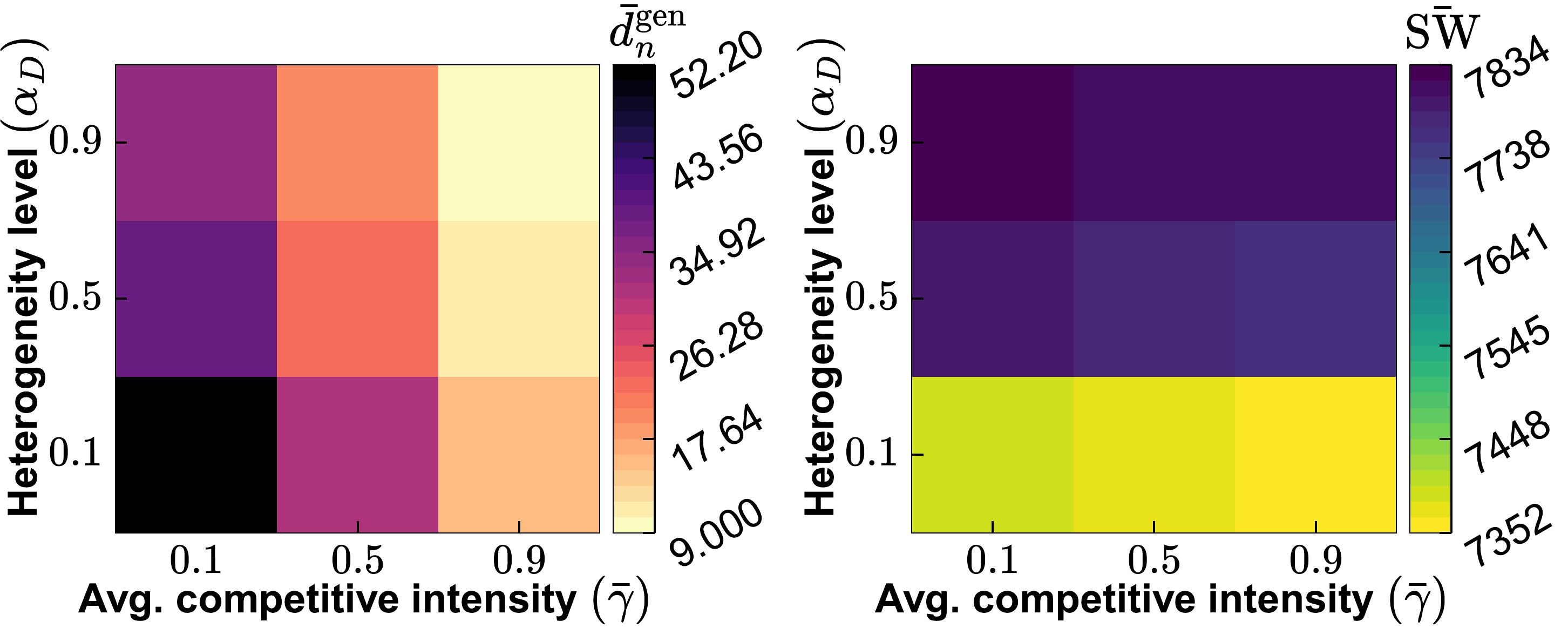}  
    \caption{FMNIST with MobileNetV3-small.}
    \label{fmnist:impacts_alphaD_gamma}
\end{subfigure} 
\begin{subfigure}[t]{0.492\textwidth}
    \centering
    \includegraphics[width=\linewidth]{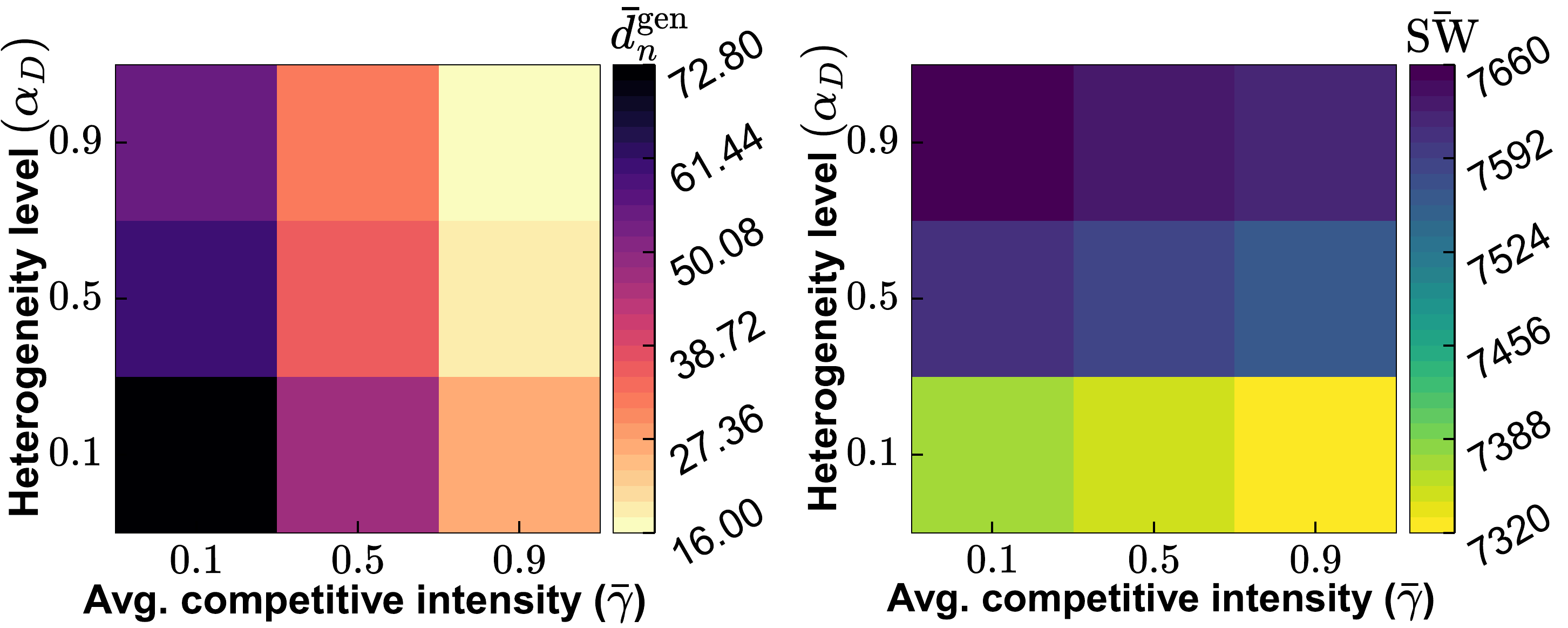}  
    \caption{CIFAR-10 with MobileNetV2.}
    \label{cifar10:impacts_alphaD_gamma}
\end{subfigure}
\begin{subfigure}[t]{0.492\textwidth}
    \centering
    \includegraphics[width=\linewidth]{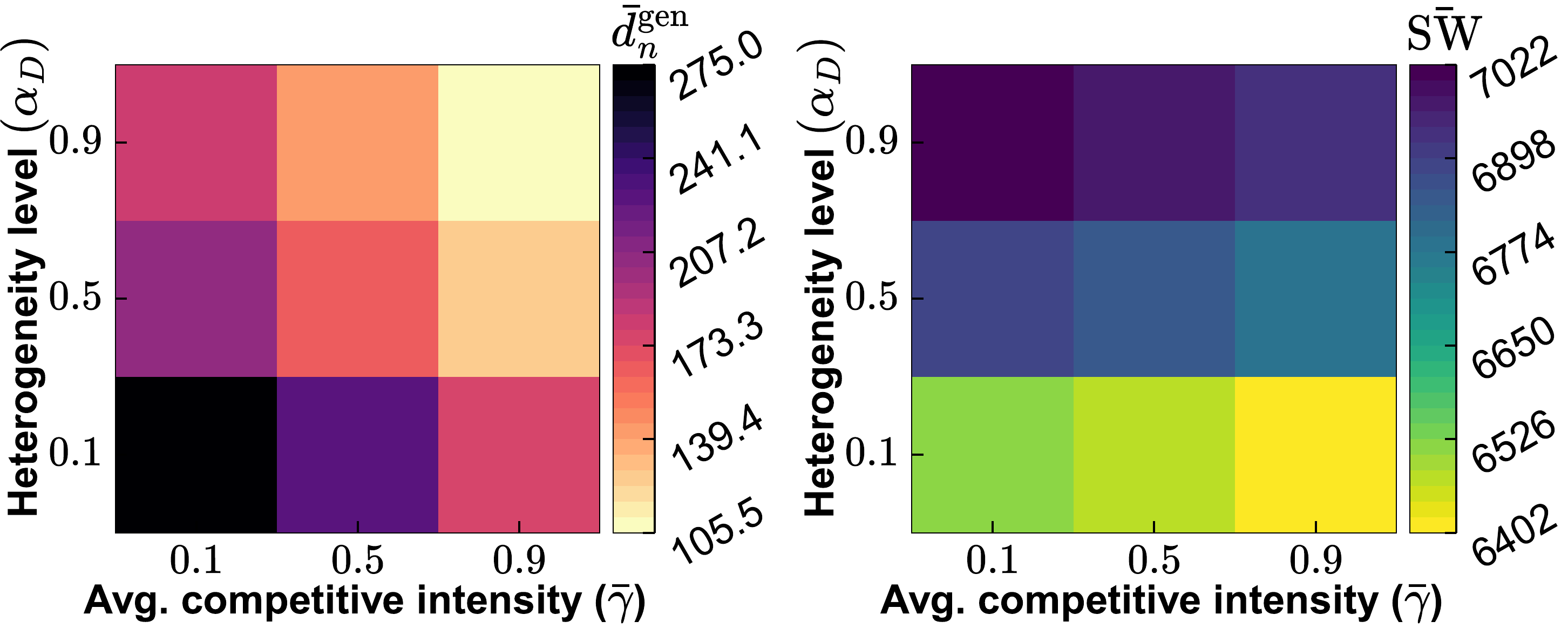}  
    \caption{CIFAR-100 with ResNet-34.}
    \label{cifar100:impacts_alphaD_gamma}
\end{subfigure} 
\caption{Impact of $\bar\gamma$ and $\alpha_{D}$ on the organization's data generation strategy and the social welfare of \proposed, where $\bar\gamma$ is scaled by $2\times10^{-3}$.}
\vspace{15pt}
\label{fig:impacts_alphaD_gamma_on_3_datasets}
\end{figure}

\paragraph{{Impacts of average competition intensity ($\bar\gamma$) and data heterogeneity ($\alpha_D$) levels}}\label{para:residual_competiton} In Fig.~\ref{fig:impacts_alphaD_gamma_on_3_datasets}, we examine the coupled effects of data heterogeneity ($\alpha_D$) and average competitive intensity ($\bar\gamma$) on 1) the average amount of GenAI-generated data among organizations and 2) the resulting social welfare, with a fixed value\footnote{{Similar trends are observed on other values of $\xi$.}} of $\xi = 90$. {{Across settings, stronger competition $\bar\gamma$ reduces both synthetic data generation and social welfare. This behavior is consistent with the residual competition coefficient (see Eq.~\eqref{eq:theta_and_c}). Under the assumption $\xi\le\min_{n'}\phi_{n'}$, payoff redistribution only compensates part of the competitive loss, while stronger bilateral competition amplifies this loss. Specifically, Eq.~\eqref{eq:solution} determines data generation through the balance between the learning-benefit term $\frac{\psi_n \epshat}{N\varrho}$, the residual competition coefficient $\theta_n$, and the unit computational cost $c_n$. A larger $\theta_n$ reduces the candidate generation volume for a given trial global loss, consistent with the lower NE generation observed under stronger $\bar\gamma$.

Under fixed competitive intensity (e.g., $\bar\gamma = 0.9\times2\times10^{-3}$) but increasing heterogeneity (i.e., smaller $\alpha_D$), organizations generate more synthetic data but the social welfare decreases. This is because additional generation incurs both generation and training costs, while its marginal coopetitive learning benefit diminishes as data volume increases, as shown in Fig.~\ref{fig:hyperparameter}. Consequently, generating more synthetic data may only partly compensate for the increase in local training loss associated with stronger data heterogeneity. The generated-data volume also differs across dataset-model pairs. Particularly, the largest average generation volume reaches $275$ samples on CIFAR-100, compared with $52.2$ on FMNIST and $72.8$ on CIFAR-10. These differences reflect the interaction between the fitted learning performance and the economic parameters governing organizational decisions.
}}


\begin{figure*}[t!]
	\centering
	\includegraphics[width=\linewidth]{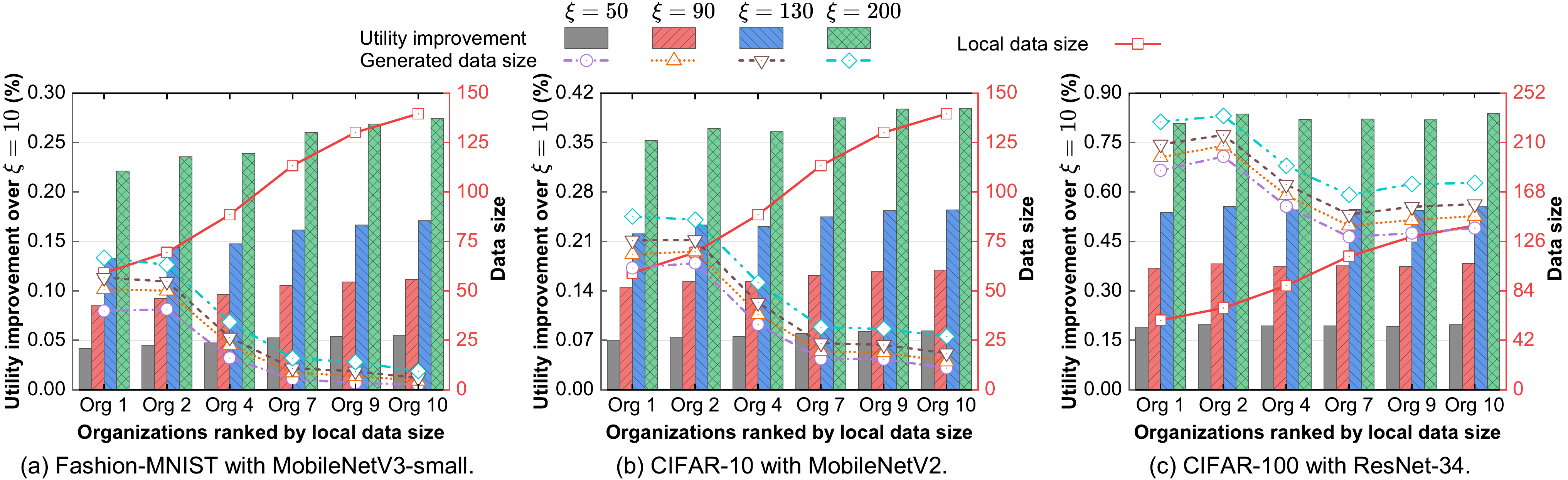}
	\caption{{Average utility improvement over $\xi=10$ and average generated-data volume for selected organizations ranked by increasing local dataset size, under different $\xi$, with $\alpha_D = 0.5$ and $\bar \gamma = 0.5\times2\times10^{-3}$ across three datasets  in {\proposed}.}}
    \label{fig:org_utility_xi}
\end{figure*} 

\paragraph{{Impact of the payoff compensation rate  $\xi$}} \label{para:impacts_xi}

{ {
$\xi$ influences the organizations' data generation decisions through the residual competition coefficient $\theta_n$. Specifically, $\partial \theta_n / \partial \xi = -\sum_{n' \neq n} \gamma_{n, {n'}} \leq 0$, so increasing the compensation rate reduces the competitive loss associated with improving an organization's local performance. Holding the other organizations' strategies fixed, this increases the incentive for data generation. For example, under moderate competition (i.e., $\bar\gamma=0.5\times2\times10^{-3}$), increasing $\xi$ from $10$ to $200$ lowers $\bar\theta_n$ from $2.63$ to $0.89$ (i.e., by $66\%$). Over the same compensation range, average data generation increases by $107\%$, $91\%$, and $33\%$ on FMNIST, CIFAR-10, and CIFAR-100, respectively. Under the same setting, the resulting generation adjustments also improve the social welfare. As $\xi$ increases, the normalized welfare gain $\omega(\xi)$ rises from approximately $45\%$, $50\%$, and $80\%$ to $81\%$, $84\%$, and $94\%$. 
}}

{ {
Fig.~\ref{fig:org_utility_xi} presents the average utility improvements relative to $\xi=10$ and the generated-data volumes of selected organizations ranked by increasing local dataset size. Organizations with smaller original datasets generally generate more synthetic data, consistent with the dependence on $\dloc$ in Eq.~\eqref{eq:solution}. Those with larger original datasets undertake less augmentation and can benefit from the improved shared model while incurring lower generation costs. Separately, payoff redistribution depend on differences in modeled local training loss through Eq.~\eqref{eq:pay_off_redistribution_1}. Because this loss reflects both original local and generated data, an organization with a small generated-data volume may still receive compensation if its modeled local loss is sufficiently low relative to its competitors. Thus, the utility improvements reflect the combined changes in cooperation gains, competitive effects, payoff redistribution, and computational costs.

Furthermore, the utility improvements are positive across $\xi=\{50, 90,130,200\}$, although their magnitudes differ among organizations. For example, the relative utility improvements for the organizations with the smallest and largest local datasets on FMNIST are approximately $0.22\%$ and $0.27\%$ corresponding to $\xi=200$, respectively. All organizations satisfy the IR constraint (cf. Constraint~\ref{def:ir}) at every $\xi$. Thus, these results show that {\proposed} improves the social welfare while maintaining nonnegative individual utilities.
}}

\subsection{\proposed { Efficiency}}\label{subsec:efficiency}

\begin{table*}[t]
\centering
\caption{Average social welfare gain over {\vcfl} (i.e., $\Delta\mathrm{SW}_m$), across competition and data-heterogeneity settings on three datasets. Higher values are better. Competition values $\bar \gamma$ are scaled by $2\times10^{-3}$. \best{Red} and \snd{green} indicate the highest and second-highest mean gains, respectively. {\wco} is an idealized competition-free reference. The {\proposed} column is shaded.}
\label{tab:social_welfare_efficiency}

\renewcommand{\arraystretch}{1.10}
\setlength{\tabcolsep}{4pt}
\resizebox{\textwidth}{!}{%
\begin{tabular}{
l|c|cccccc|
>{\columncolor{gray!10}}c
}
\toprule
\rowcolor{gray!20}
\textbf{Parameter} & \textbf{Setting}
& \textbf{{\wpr}} & \textbf{\wco} & \textbf{\wdg}
& \textbf{{\radg}} & \textbf{{\madg}}
& \textbf{Fixed-half} & \textbf{\proposed} \\
\midrule

\rowcolor{gray!20}
\multicolumn{9}{c}{\textit{FMNIST}} \\
\midrule

\multirow{3}{*}{$\bar\gamma$ (fixed $\alpha_D=0.5$)}
& $0.1$
& $48.91 \pm 14.54$
& \best{$55.06 \pm 15.05$}
& $-0.01 \pm 0.02$
& $-814.95 \pm 63.73$
& $-1990.88 \pm 123.14$
& $-791.66 \pm 63.66$
& \snd{$50.79 \pm 14.67$} \\

& $0.5$
& $24.74 \pm 11.53$
& \best{$55.06 \pm 15.05$}
& $-0.03 \pm 0.11$
& $-814.95 \pm 63.73$
& $-1990.88 \pm 123.14$
& $-791.66 \pm 63.66$
& \snd{$33.42 \pm 12.93$} \\

& $0.9$
& $7.08 \pm 5.29$
& \best{$55.06 \pm 15.05$}
& $-0.05 \pm 0.18$
& $-814.95 \pm 63.73$
& $-1990.88 \pm 123.14$
& $-791.66 \pm 63.66$
& \snd{$17.99 \pm 8.64$} \\

\addlinespace[2pt]

\multirow{2}{*}{$\alpha_D$ (fixed $\bar\gamma=0.5$)}
& $0.1$
& $35.42 \pm 14.18$
& \best{$73.27 \pm 17.36$}
& $-0.03 \pm 0.10$
& $-659.52 \pm 65.27$
& $-1793.65 \pm 124.51$
& $-625.05 \pm 65.70$
& \snd{$46.53 \pm 15.24$} \\

& $0.9$
& $21.74 \pm 10.53$
& \best{$49.70 \pm 14.19$}
& $-0.03 \pm 0.10$
& $-838.61 \pm 63.51$
& $-2018.81 \pm 122.98$
& $-816.72 \pm 63.35$
& \snd{$29.68 \pm 12.06$} \\
\midrule

\rowcolor{gray!20}
\multicolumn{9}{c}{\textit{CIFAR-10}} \\
\midrule

\multirow{3}{*}{$\bar\gamma$ (fixed $\alpha_D=0.5$)}
& $0.1$
& $78.66 \pm 18.89$
& \best{$87.01 \pm 19.18$}
& $-0.01 \pm 0.02$
& $-639.78 \pm 65.64$
& $-1774.93 \pm 124.66$
& $-604.95 \pm 66.17$
& \snd{$81.24 \pm 18.93$} \\

& $0.5$
& $43.11 \pm 16.17$
& \best{$87.01 \pm 19.18$}
& $-0.03 \pm 0.12$
& $-639.78 \pm 65.64$
& $-1774.93 \pm 124.66$
& $-604.95 \pm 66.17$
& \snd{$56.43 \pm 17.25$} \\

& $0.9$
& $14.09 \pm 8.48$
& \best{$87.01 \pm 19.18$}
& $-0.05 \pm 0.19$
& $-639.78 \pm 65.64$
& $-1774.93 \pm 124.66$
& $-604.95 \pm 66.17$
& \snd{$32.61 \pm 12.97$} \\

\addlinespace[2pt]

\multirow{2}{*}{$\alpha_D$ (fixed $\bar\gamma=0.5$)}
& $0.1$
& $58.14 \pm 19.10$
& \best{$107.09 \pm 21.47$}
& $-0.03 \pm 0.12$
& $-570.25 \pm 66.65$
& $-1691.28 \pm 125.39$
& $-531.31 \pm 67.44$
& \snd{$73.75 \pm 20.05$} \\

& $0.9$
& $37.63 \pm 14.94$
& \best{$77.86 \pm 18.11$}
& $-0.03 \pm 0.11$
& $-689.61 \pm 65.04$
& $-1837.19 \pm 124.18$
& $-658.11 \pm 65.41$
& \snd{$49.64 \pm 16.05$} \\
\midrule

\rowcolor{gray!20}
\multicolumn{9}{c}{\textit{CIFAR-100}} \\
\midrule

\multirow{3}{*}{$\bar\gamma$ (fixed $\alpha_D=0.5$)}
& $0.1$
& $399.42 \pm 50.21$
& \best{$413.43 \pm 50.07$}
& $-0.01 \pm 0.04$
& $-51.53 \pm 82.16$
& $-1127.34 \pm 134.67$
& $7.75 \pm 84.64$
& \snd{$403.78 \pm 50.08$} \\

& $0.5$
& $330.11 \pm 52.49$
& \best{$413.43 \pm 50.07$}
& $-0.06 \pm 0.22$
& $-51.53 \pm 82.16$
& $-1127.34 \pm 134.67$
& $7.75 \pm 84.64$
& \snd{$359.27 \pm 50.75$} \\

& $0.9$
& $237.84 \pm 48.81$
& \best{$413.43 \pm 50.07$}
& $-0.10 \pm 0.37$
& $-51.53 \pm 82.16$
& $-1127.34 \pm 134.67$
& $7.75 \pm 84.64$
& \snd{$305.67 \pm 49.28$} \\

\addlinespace[2pt]

\multirow{2}{*}{$\alpha_D$ (fixed $\bar\gamma=0.5$)}
& $0.1$
& $476.44 \pm 63.75$
& \best{$574.73 \pm 61.36$}
& $-0.06 \pm 0.24$
& $327.52 \pm 93.44$
& $-671.10 \pm 143.28$
& $408.51 \pm 96.42$
& \snd{$510.25 \pm 61.96$} \\

& $0.9$
& $278.91 \pm 48.00$
& \best{$357.17 \pm 45.71$}
& $-0.05 \pm 0.22$
& $-176.95 \pm 78.47$
& $-1274.49 \pm 132.09$
& $-124.45 \pm 80.72$
& \snd{$306.41 \pm 46.41$} \\

\bottomrule
\end{tabular}%
}
\end{table*}

{ {
Table~\ref{tab:social_welfare_efficiency} evaluates the average social welfare gains over {\vcfl} across competition and data heterogeneity settings, with $\xi = 90$ for {\proposed}. {\proposed} achieves positive mean welfare gains and outperforms {\wpr}, \wdg, {\radg}, {\madg}, and Fixed-half. {\wco} achieves the largest mean gains under its competition-free assumption and serves as a reference for quantifying the welfare gap in the absence of competition.

\paragraph{Impact of data heterogeneity ($\alpha_D$)} {With a fixed\footnote{{Similar trends are observed on other values of $\bar\gamma= \{0.1, 0.9\}$.}} moderate competition $\bar\gamma = 0.5\times2\times10^{-3}$, {\proposed} achieves larger social welfare gains over {\vcfl} as data heterogeneity increases. Specifically, from $\alpha_D=0.9$ to $\alpha_D=0.1$, its mean welfare gains rise by approximately $56.77\%$, $48.57\%$, and $66.53\%$ on FMNIST, CIFAR-10, and CIFAR-100, respectively, emphasizing the important role of GenAI-based data augmentation under strong data imbalance settings. The comparisons with {\radg}, {\madg}, and Fixed-half highlight the importance of balancing coopetitive learning benefits against generation and training costs. {\madg} produces negative mean welfare gains across all settings, indicating that maximum data generation can be economically inefficient. {\radg} and Fixed-half also yield negative mean gains in most settings, but become beneficial on CIFAR-100 under strong heterogeneity (i.e., $\alpha_D=0.1$), achieving mean gains of $327.52$ and $408.51$, respectively. In this setting, {\proposed} achieves a mean welfare gain of $510.25$, exceeding these values by approximately $55.79\%$ and $24.91\%$, respectively. These results show that selecting generation volumes for individual organizations provides a welfare advantage over the random and fixed-volume strategies, even when substantial data augmentation is beneficial.

Besides, payoff redistribution remains beneficial across $\alpha_D$ settings. Specifically, {\proposed} achieves approximately $28$--$35\%$ recovered welfare between {\wpr} and the competition-free {\wco} reference. This improvement shows that optimizing data generation without compensation leaves a significant competitive loss unresolved. By mitigating this loss, redistribution changes NE data generation decisions and improves the welfare beyond {\wpr}. 

\paragraph{Impact of competitive intensity ($\bar\gamma$)} With $\alpha_D = 0.5$ and varying competitive intensity levels, {\proposed} maintains positive welfare gains over {\vcfl}, although these gains decrease as competition intensifies. Increasing $\bar\gamma$ from $0.1\times2\times10^{-3}$ to $0.9\times2\times10^{-3}$ reduces its mean gain from $50.79$ to $17.99$ on FMNIST, from $81.24$ to $32.61$ on CIFAR-10, and from $403.78$ to $305.67$ on CIFAR-100, corresponding to reductions of approximately $64.58\%$, $59.86\%$, and $24.30\%$, respectively. The trend is consistent with stronger residual competition weakening data-generation incentives, as discussed in Section~\ref{subsec:feasibility}(\ref{para:residual_competiton}). However, the welfare advantage of {\proposed} over {\wpr} increases with competition. On CIFAR-100, for example, the advantage grows from $4.36$ under low $\bar\gamma$ to $67.83$ under high $\bar\gamma$. A fixed increase in $\xi$ reduces $\theta_n$ by an amount proportional to the organization’s aggregate competitive intensity, providing a direct explanation for why compensation can have a larger effect when $\bar\gamma$ is stronger. Furthermore, the fraction of the {\wco}-{\wpr} welfare gap recovered exhibits different trends across datasets. It decreases from approximately $30.57\%$ to $22.74\%$ on FMNIST, from $30.90\%$ to $25.40\%$ on CIFAR-10, while increasing from $31.12\%$ to $38.63\%$ on CIFAR-100. Thus, payoff redistribution partially mitigates competition-caused welfare losses, with a particularly pronounced benefit on difficult learning tasks (i.e., CIFAR-10, CIFAR-100) under strong competition.
}}}

\subsection{{Discussion and Key Findings}}\label{subsec:discussion}
The experimental results yield three principal findings that directly address the RQs posed in Section~\ref{sec:intro}.

{
\paragraph{Coupled effects of data heterogeneity ($\alpha_D$) and competition ($\gamma$) (\textbf{\#RQ1})}\label{para:rq1}

Fig.~\ref{fig:impacts_alphaD_gamma_on_3_datasets} and Table~\ref{tab:social_welfare_efficiency} show that stronger data heterogeneity (i.e., $\alpha_D=0.1$) increases synthetic-data generation but lowers social welfare. However, the social welfare gain over {\vcfl} increases because the social welfare of {\vcfl} decreases more sharply. This indicates that adaptive data augmentation becomes more valuable relative to {\vcfl} when local datasets are more imbalanced. Stronger competition (i.e., $\bar\gamma=0.9\times2\times10^{-3}$) reduces generation volumes and welfare gains by weakening generation incentives through residual competition $\theta_n$. \textit{These findings indicate that generation decisions should account jointly for the learning benefits of data augmentation, computational costs, and competitive effects.}

\paragraph{Strategic data generation through NE characterization (\textbf{\#RQ2})}\label{para:rq2} Theorems~\ref{def:weighted_pg} and~\ref{def:theorem_2} establish the weighted potential function and the unique NE of the continuous relaxation, whose scalar characterization enables computation by bisection. The resulting strategies account for each organization’s collaborative learning benefit, data generation and training cost (i.e., $c_n$) and residual competition (i.e., $\theta_n$). Empirically, {\proposed} achieves higher mean social welfare than {\wpr}, \wdg, {\radg}, {\madg}, and Fixed-half across settings. For example, {\madg} produces negative welfare gains over {\vcfl} throughout the evaluated scenarios, illustrating that excessive generation can be economically inefficient. These results support selecting data generation volumes according to each organization’s cooperation gains, computation costs, and residual competition, consistent with Appendix~\ref{subsec:param_est}, with the resulting welfare improvements established empirically.

\paragraph{Budget-balanced payoff redistribution (\textbf{\#RQ3})}\label{para:rq3} The $\xi$-scaled redistribution mechanism internalizes competitive externalities and improves the social welfare by adjusting organizations' data-generation incentives. Because transfers sum to zero (see Remark~\ref{remark:payoff_red}), they do not directly increase social welfare at a strategy profile. Instead, increasing compensation rate $\xi$ lowers residual competition, encourages additional generation, and improves the social welfare gains across three datasets (cf. Fig.~\ref{fig:org_utility_xi}). Comparisons with {\wpr} (cf. Table~\ref{tab:social_welfare_efficiency}) further show that redistribution provides a larger welfare benefit under stronger competition, while the largest recovered welfare fraction occurs on CIFAR-100 under strong competition.
}
\section{Conclusion}\label{sec:conclusion}

In this paper, we proposed {\proposed}, a framework for strategic GenAI-based data generation in CFL under data heterogeneity and inter-organizational competition. We derived a weighted potential function, established the existence and uniqueness of the continuous NE, and developed a provably convergent scalar bisection algorithm. Extensive experimental results showed that stronger competition reduces data generation and social welfare, while greater data heterogeneity increases generation demand but lowers welfare. Besides, increasing the compensation rate mitigates residual competition, encourages generation, and improves welfare. These findings highlight the importance of balancing learning benefits, computational costs, and competitive effects when determining data generation strategies. {The present analysis focuses on organizations’ data generation decisions within a single training round, thereby evaluating the resulting model-training dynamics and convergence across successive FL rounds remains future work.}

\section*{Acknowledgment}
This publication has emanated from research conducted with the financial support of Research Ireland under Grant number 18/CRT/6222; School of Computer Science and Statistics, Trinity College Dublin. The work of Quoc-Viet Pham is supported in part by Research Ireland under the European Innovation Council CHIST-ERA SHIELD project (Project No. 216449, Award No. 19226). For the purpose of Open Access, the author has applied a CC BY public copyright licence to any Author Accepted Manuscript version arising from this submission.

\appendix

For clarity and readability of the proofs, we restate defined assumptions and conditions as follows: 
\begin{align}
    & \alpha > 0, \;\; \beta > 0, \;\; \delta \ge 0, \;\; \varrho > 0, \;\;\psi_n > 0, \;\; \dmix = \dloc + \dgen, \nonumber\\
    & \phi_{n'} \geq \xi, \;\; \gamma_{n, {n'}} \geq 0, 0 \leq \dgenmin < \dgenmax < \infty, \;\; \dloc + \dgenmin > 0, \nonumber\\
    & c_n > 0, \;\; \mathcal{D} = [\dgenmin, \dgenmax]^N \nonumber.
\end{align}

\subsection{Proof of Theorem 1}
\label{proof:theorem1}
We adopt the forward method \cite{la2016potential}, where we consider a unilateral deviation from $\dbold = \left(\dgen,\boldsymbol{d}^{\text{gen}}_{-n}\right)$ to $\dboldscript=\left(\dgenscript,\boldsymbol{d}^{\text{gen}}_{-n}\right)$ (i.e., the other organizations' strategy profile $\boldsymbol{d}_{{-n}}^{{\text{gen}}}$ keeps unchanged). Throughout, 
{
\begin{align}
    {\theta_n} = \sum_{n' \neq n} \gamma_{n, {n'}} (\phi_{n'} - \xi) \geq 0, \quad c_n = \kappa_nC_n^{\text{cmp}} (\eta_n  +\mu_n) f_n^2.
\end{align}}
\noindent First,  taking the difference between organization $n$'s utilities in Eq.~\eqref{eq:client_utility_one_round}  results in following Eq.~\eqref{eq:proof_weighted_pg1}, shown at the bottom of the page. Besides, we define the following potential-minimization objective function: 
\begin{align}
\label{eq:fgen}
\F= \esilontotaldgen - \sum_{n \in \mathcal{N}} \frac{\theta_n}{\psi_n}\epsilon_n(\dgen)  + \sum_{n \in \mathcal{N}} \frac{c_n}{\psi_n} \dgen.
\end{align}

\begin{figure*}[b]
\noindent\makebox[\linewidth]{\rule{\textwidth}{0.4pt}}
\begin{align}
\label{eq:proof_weighted_pg1}
 & {U}_{n}(d^\text{\text{gen}}_{n}, \boldsymbol{d}^\text{\text{gen}}_{\boldsymbol{-n}}) - {U}_{n}(\dgenscript, \boldsymbol{d}^\text{\text{gen}}_{\boldsymbol{-n}})  \nonumber\\ 
 & = -\psi_n \left[\esilontotaldgen - \epsilon_0 \right]  - \sum_{n' \neq n} \gamma_{n, {n'}} (\phi_{n'}- \xi) \left[\esilondgenother - \esilondngen \right] 
 -  \kappa_n C_n^{\text{cmp}} \left[(\eta_n   + \mu_n)\dgen + \eta_n \dloc\right] f_n^2 \nonumber \\ & + \psi_n \left[\esilontotaldgenscript - \epsilon_0 \right] + \sum_{n' \neq n} \gamma_{n, {n'}} (\phi_{n'} - \xi) \left[\esilondgenother - \esilondngenscript \right] 
 + \kappa_n C_n^{\text{cmp}} \left[(\eta_n   + \mu_n)\dgenscript + \eta_n d_{n}^{\text{loc}}\right] f_n^2 
\nonumber\\
& {= -\psi_n \left[\esilontotaldgen - \esilontotaldgenscript\right] +  \sum_{n' \neq n} \gamma_{n, {n'}} (\phi_{n'} - \xi)\left[\esilondngen - \esilondngenscript\right] - \kappa_n C_n^{\text{cmp}}(\eta_n   + \mu_n)(\dgen - \dgenscript)f_n^2}
\nonumber\\
& = { -\psi_n \left[\esilontotaldgen - \esilontotaldgenscript \right] + \theta_n \left[\esilondngen - \esilondngenscript\right]
- c_n(\dgen  - \dgenscript).}
\end{align}
\end{figure*}

\noindent For the same unilateral deviation, we have the following equality: 
\begin{align}
\label{eq:delta_f}
    \F - \Fscript = & \;\;\esilontotaldgen - \esilontotaldgenscript \nonumber\\  
    & - \frac{\theta_n}{\psi_n}\left[\esilondngen - \esilondngenscript\right] \nonumber\\
    & + \frac{c_n}{\psi_n}(\dgen - \dgenscript).
\end{align}
Based on~\Cref{eq:fgen,eq:proof_weighted_pg1,eq:delta_f}, we have: 
\begin{align}
\label{eq:weighted_pg_1}
    {U}_{n}(\dbold) - {U}_{n}(\dboldscript) = -\psi_n\left[\F - \Fscript\right].
\end{align}

\noindent Thus, with $\Phi = - F$, Eq.~\eqref{eq:weighted_pg_1} becomes
\begin{align}
\label{eq:weighted_pg}
        {U}_{n}(\dbold) - {U}_{n}(\dboldscript) = \psi_n\left[\Phi(\dbold) - \Phi(\dboldscript)\right].
\end{align}

Since $\psi_n > 0$, \textit{the game $\mathcal{G}$ is a weighted potential game with potential function $\Phi = -F$ and positive weights $\psi_n$}. Eq.~\eqref{eq:weighted_pg} shows that every unilateral utility improvement corresponds to a decrease in the potential-minimization objective $F$, with $\psi_n$ determining the scale of the utility change. Therefore, each organization's best-response problem is equivalent to minimizing $F$ with respect to its own strategy while the other organizations' strategies remain fixed.

\subsection{Proof of Theorem 2}
\label{proof:theorem2}
We first establish Nash equilibrium existence through a global minimizer of $F$. We then derive necessary equilibrium conditions and show that they reduce to a single equation for the equilibrium global loss with a unique solution. Consequently, these results establish the unique NE and the unique global minimizer of $F$. Finally, we prove convergence of the scalar bisection procedure.
\paragraph{Existence of a Nash equilibrium (NE)}
\label{proof:ne_existence}
{{
The potential-minimization objective $F$ is continuous on the nonempty compact set $\mathcal{D}$, so it has at least one global minimizer. \textit{At a global minimizer of $F$}, no feasible unilateral deviation can decrease $F$. By Theorem~\ref{def:weighted_pg}, an organization's utility increases  under a unilateral deviation if and only if $F$ decreases. Therefore, no organization can improve its utility by deviating unilaterally, and every global minimizer of $F$ is an NE.

\paragraph{Necessary optimality conditions}\label{proof:necess_con}
The local-loss derivative and the corresponding global-loss derivative are
\begin{align}
\label{eq:der_local_global}
    & \esdnfirder = -\alpha\beta(\dloc + \dgen)^{-\beta-1} < 0, \nonumber \\ 
    & \frac{\partial \esilontotaldgen}{\partial \dgen} = \frac{\esilontotaldgen}{N\varrho} \esdnfirder <0.
\end{align}
\noindent Both derivatives are negative. Thus, \textit{increasing an organization’s generated data volume reduces its modeled local loss and the modeled global loss}.

\noindent The first-order derivative of $F$ is as follows:
\begin{align}
    \label{eq:1st_derivative_1} \frac{\partial \F}{\partial \dgen} &= \frac{\esilontotaldgen}{N\varrho} \esdnfirder - \esdnfirder\frac{\theta_n}{\psi_n} + \frac{c_n}{\psi_n} \nonumber\\ & = \esdnfirder\left[\frac{\esilontotaldgen}{N\varrho} - \frac{\theta_n}{\psi_n}\right] + \frac{c_n}{\psi_n} \nonumber \\
    &= \frac{H_n(\dbold)}{\psi_n}.
\end{align}

At any NE, each organization minimizes $F$ over its own data-generation interval while the other strategies remain fixed. A necessary condition is therefore a nonnegative derivative at the lower bound, a zero derivative in the interior, or a nonpositive derivative at the upper bound. These conditions can also be expressed using the KKT conditions for the data-generation bounds. The corresponding Lagrangian is as follows:
\begin{align}
\label{eq:lagrangian}
    \mathcal{L} = & \F + \sum_n \lambda_n(\dgenmin-\dgen) + \sum_n \upsilon_n(\dgen - \dgenmax),
\end{align}
with primal feasibility, dual feasibility, stationarity, and complementary slackness, presented by, respectively,
\begin{subequations}
\label{eq:kkt_conditions}
\begin{empheq}[left={\text{KKT conditions}}\empheqlbrace]{align}
    & \dgenmin \leq \dgenopt \leq \dgenmax, \nonumber\\
    & \lambda_n \ge 0, \;\upsilon_n \ge 0, \nonumber\\
    & \frac{\partial F(\dopt)}{\partial\dgen} - \lambda_n + \upsilon_n = 0, \nonumber\\ 
    &\lambda_n(\dgenopt - \dgenmin)=0, \; \upsilon_n(\dgenmax-\dgenopt) = 0. \nonumber
\end{empheq}       
\end{subequations}
Next, we derive each boundary condition as follows:
\begin{itemize}
    \item At the lower bound, $\upsilon_n=0$. Consequently, $\frac{\partial F(\dopt)}{\partial\dgen} = \lambda_n \geq 0$.
    \item In the interior, $\lambda_n=\upsilon_n=0$. Consequently, $\frac{\partial F(\dopt)}{\partial\dgen} = \lambda_n = 0$.
    \item At the upper bound, $\lambda_n=0$. Consequently, $\frac{\partial F(\dopt)}{\partial\dgen} = -\upsilon_n \leq 0$.
\end{itemize}

According to Eq.~\eqref{eq:1st_derivative_1}, with $\psi_n > 0$, $H_n$ and $\frac{\partial F}{\partial\dgen}$  have the same sign, as a result,
\begin{subequations}
\label{eq:boundary_conditions}
\begin{empheq}[left={}\empheqlbrace]{align}
        & H_n(\dopt) \geq 0, \;\;\; \dgenopt = \dgenmin, \\ 
       \label{eq:interior}&H_n(\dopt) = 0, \;\;\; \dgenmin < \dgenopt < \dgenmax, \\ 
       &H_n(\dopt) \leq 0, \;\;\; \dgenopt = \dgenmax.
\end{empheq}       
\end{subequations}
Every NE satisfies these conditions because no organization can improve its utility by changing its own data-generation strategy while the others’ strategies remain fixed. Their sufficiency will follow from the unique scalar characterization established below.

\paragraph{Scalar characterization of the optimality conditions}
\label{proof:sca_charac} 

We introduce $\hat{\epsilon}>0$ as a trial global-loss value, and define:
\begin{align}
   a_n(\ehat) = \frac{\psi_n\ehat}{N\varrho} - \theta_n.
\end{align}
For this trial value, based on Eq.~\eqref{eq:1st_derivative_1}, we have
\begin{equation}
    \widehat{H}_n(v,\hat{\epsilon})
    = c_n-\alpha\beta(\dloc+v)^{-\beta-1}
    a_n(\hat{\epsilon}),
\end{equation}
where $v$ denotes organization $n$'s data-generation decision.
The actual derivative expression satisfies:
\begin{align}
    H_n(\dbold)
    =
    \widehat{H}_n\!\left(
        \dgen, \epsilon(\dbold)
    \right).
\end{align}

We construct a candidate strategy by applying the boundary
sign conditions to $\widehat{H}_n(\cdot,\hat{\epsilon})$.
These coincide with the necessary NE conditions
when $\hat{\epsilon} =\epsilon(\dbold)$.

\noindent \textbf{If $a_n(\ehat) \leq 0$ (i.e., competition losses outweigh collaborative learning benefits)}, then 
\begin{align}
    \widehat{H}_n(\dgen,\hat{\epsilon}) = c_n -\alpha\beta(\dloc + \dgen)^{-\beta-1}a_n(\ehat) \geq c_n > 0.
\end{align}

Thus, only the lower-bound condition in~\eqref{eq:boundary_conditions} can hold, and we have the following data-generation volume candidate.
\begin{align}
    \dgen = \dgenmin.
\end{align}
\noindent \textbf{If $a_n(\ehat) >  0$ (i.e., positive net collaborative learning benefits)}, then, we define 
\begin{align}
    \label{eq:q_n}
    q_n(\ehat) = \left[\frac{\alpha\beta}{c_n}a_n(\ehat)\right]^{1/(\beta+1)} - \dloc.
\end{align}

This candidate is obtained by solving the interior condition $H_n=0$ for a given global-loss value $\ehat$. Its definition implies
\begin{align}
\label{eq:substi}
    \alpha\beta a_n(\ehat) = c_n(\dloc +q_n(\ehat))^{\beta+1}.
\end{align}
Substituting Eq.~\eqref{eq:substi} into $\widehat{H}_n$, we have the following equality:
\begin{align}
\label{eq:h_n}
        \widehat{H}_n(\dgen,\hat{\epsilon}) = c_n \left[1 - \left(\frac{\dloc + q_n(\ehat)}{\dloc + \dgen}\right)^{\beta+1}\right].
\end{align}

Since $c_n >0$ and $\beta+1 >0$, Eq.~\eqref{eq:h_n}'s sign depends on the following conditions:

\begin{itemize}
    \item If $q_n(\ehat) \leq \dgenmin$,  we select $\dgen(\ehat) = \dgenmin\geq q_n(\ehat)$. Thus, $0 <\frac{\dloc + q_n(\ehat)}{\dloc + \dgen\left(\ehat\right)} \leq 1$, thereby $\widehat{H}_n(\dgenmin,\hat{\epsilon}) \geq 0$, which is required for the lower-bound condition,
    \item If $\dgenmin < q_n(\ehat) < \dgenmax$, we select $\dgen(\ehat) = q_n(\ehat)$. Thus, $\widehat{H}_n(q_n(\ehat),\hat{\epsilon}) = 0$ as required for an interior strategy, 
    \item If $q_n(\ehat) \geq \dgenmax$, we select $\dgen(\ehat) = \dgenmax\leq q_n(\ehat)$. Thus, $\frac{\dloc + q_n(\ehat)}{\dloc + \dgen\left(\ehat\right)} \geq 1$, thereby $\widehat{H}_n(\dgenmax,\hat{\epsilon}) \leq 0$, which is required for the upper-bound condition. 
\end{itemize}

These cases achieve the projected candidate mapping
$\dbold(\hat{\epsilon})$.
At any NE, the trial value must equal the actual global loss.
Consequently, every NE must satisfy
\begin{align}
\dbold=\dbold(\hat{\epsilon}),\quad\hat{\epsilon}=\epsilon\!\left(\dbold(\hat{\epsilon})\right).
\end{align}
We therefore define the following residual function:
\begin{align}
h(\hat{\epsilon})=\epsilon\!\left(\dbold(\hat{\epsilon})\right)-\hat{\epsilon}.
\end{align}
Every NE corresponds to a zero of $h$.
The following step establishes the unique zero and
identifies its associated candidate profile as the unique NE. We can construct each candidate strategy as follows:
\begin{align}
\label{eq:derived_solution}
\dgen(\ehat) =
\begin{cases}
    \dgenmin, & a_n(\ehat) \leq 0, \\
    \Pi_{[\dgenmin, \dgenmax]}\left( \left[ \frac{\alpha\beta}{c_n} a_n(\ehat) \right]^{1/(\beta+1)} - \dloc \right), & a_n(\ehat) > 0.
\end{cases}
\end{align}

\paragraph{Existence and uniqueness of the scalar solution}
\label{proof:exist_sol}
We set 
\begin{align}
\label{eq:initial_interval}
    \ell_0= \epsilon(\dgenmax\mathbf{1}), \;\; r_0 = \epsilon(\dgenmin\mathbf{1}),
\end{align}
where $\mathbf{1}$ is the vector of $N$ ones. Besides, $\esilondngen$ decreases with $\dgen$ (see Eq.~\eqref{eq:local_loss_error}) and the exponential is increasing in each local loss. Furthermore, projections in Eq.~\eqref{eq:derived_solution} ensure that $\dbold\left(\ehat\right) \in [\dgenmin,\dgenmax]^N$, as a result, every feasible profile satisfies
\begin{align}
\label{eq:H_endpoint}
    & \ell_0 \leq \epsilon\left(\dbold\left(\ehat\right)\right)\leq r_0, \nonumber \\
    & h(\ell_0)= \epsilon\!\left(\dbold(\ell_0)\right)-\ell_0 \geq 0,\; h(r_0) = \epsilon\!\left(\dbold(r_0)\right)-r_0 \leq 0. 
\end{align}

\noindent \textit{We next show that $h$ is continuous and strictly decreasing. With the endpoint inequalities above, these properties establish the existence and uniqueness of its zero.}

Since $a_n(\ehat)$ is continuous and strictly increasing in $\ehat$, the candidate strategy has the following properties:
\begin{itemize}
    \item When $a_n(\ehat)>0$, $q_n(\ehat)$ in Eq.~\eqref{eq:q_n} is continuous and strictly increasing in $\ehat$. Projection onto the feasible interval preserves continuity and makes the resulting strategy non-decreasing, allowing it to remain constant at either bound.
    
    \item When $a_n(\ehat)\leq0$, the candidate strategy equals $\dgenmin$. At the transition between the two branches,
    \begin{align}
        \lim_{a_n(\ehat)\to0^+}q_n(\ehat)=-\dloc<\dgenmin.
    \end{align}
    Hence, the projected positive-case candidate also equals $\dgenmin$ sufficiently close to the transition. \textit{The two cases therefore join continuously or the candidate strategy is also continuous at the transition $a_n(\ehat)=0$}.
\end{itemize}
Consequently, $\dgen(\ehat)$ is continuous and non-decreasing in $\ehat$ for every organization $n\in\mathcal{N}$.

Since each local loss $\epsilon_n$ decreases with its organization's generated data volume, the global loss $\epsilon\left(\dbold\left(\ehat\right)\right)$ is continuous and non-increasing in $\ehat$. Thus, for $\ehat_2 >\ehat_1$,
\begin{align}
    h(\ehat_2) - h(\ehat_1) = \epsilon\left(\dbold\left(\ehat_2\right)\right) - \epsilon\left(\dbold\left(\ehat_1\right)\right) - (\ehat_2 - \ehat_1) < 0.
\end{align}
Therefore, $h$ is strictly decreasing even when the candidate strategies remain constant at their bounds, because the subtracted scalar $\ehat$ continues to increase. Together with continuity and Eq.~\eqref{eq:H_endpoint}, this establishes a unique zero $\ehat^*$ in the stated global-loss interval, satisfying $h(\ehat^*) = 0$.

Besides, $F$ admits at least one global minimizer on $[\dgenmin, \dgenmax]^N$, and every global minimizer is an NE. Moreover, we established that every NE must satisfy $\dbold = \dbold(\ehat)$, where $h(\ehat) = 0$. Since this scalar equation has a unique solution and the candidate mapping in Eq.~\eqref{eq:solution}  is single-valued, every NE must equal $\dbold(\ehat^*)$. Consequently, $\dopt = \dbold(\ehat^*)$ is the unique NE. Because every global minimizer is an NE and at least one global minimizer exists, this profile is also the unique global minimizer of $F$ on $[\dgenmin, \dgenmax]^N$.

\paragraph{Convergence of bisection}
\label{proof:converge_bis}

According to Eq.~\eqref{eq:initial_interval}, we have an initial interval bounded by $\ell_0$ and $r_0$, which denote the smallest and largest global losses over the strategy constraints. We have the unique solution $\ehat^*$ such that $h(\ehat^*) = 0$ lies in $[\ell_0,r_0]$. If either endpoint has zero residual, that endpoint is the solution and the algorithm terminates.

Otherwise, at iteration $k$, we set a midpoint value $\ehat_k=(\ell_k+r_k)/2$. If $h(\ehat_k) = 0$, the root has been found. Otherwise, we update the interval with the following rule:
\begin{align}
    (\ell_{k+1}, r_{k+1}) = 
    \begin{cases}
(\ehat_k, r_k), & h(\ehat_k) > 0, \\
(\ell_k, \ehat_k), & h(\ehat_k) < 0.
\end{cases}
\end{align}
Because $h$ is strictly decreasing, $h(\ehat_k) > 0$ implies $\ehat^* > \ehat_k$, so the root lies to the right, and vice versa.  Each update therefore retains the unique root and halves the interval width. Besides, its width satisfies
\begin{align}
    r_k - \ell_k = \frac{r_0 - \ell_0}{2^k}, \;\; |\ehat_k - \ehat^*| \leq \frac{r_k - \ell_k}{2} = \frac{r_0-\ell_0}{2^{k+1}}.
\end{align}
This width bound tends to zero, proving that $\ehat_k \rightarrow\ehat^*$. Furthermore, each organization candidate strategy is continuous, including where it reaches a bound or switches between its two defining cases as shown in~\eqref{eq:derived_solution}. Consequently,
\begin{align}
    \lim_{k\rightarrow\infty} \dgen(\ehat_k) = \dgen(\ehat^*) = \dgenopt, \forall n \in \mathcal{N}.
\end{align}
Thus, the candidate strategy profile converges to the unique NE of the continuous game $\mathcal{G}$. We pre-define a scalar tolerance $\epsilon_{\text{tol}} >0$, the algorithm terminates when 
\begin{align}
    \frac{r_k-\ell_k}{2} = \frac{r_0-\ell_0}{2^{k+1}} \leq \epsilon_{\text{tol}},
\end{align}
which guarantees that the returned midpoint differs from $\ehat^*$ by at most $\epsilon_{\text{tol}}$. For $r_0>\ell_0$, the stopping condition is satisfied after at most 
\begin{align} K_{\mathrm{req}} = \max\left\{ 0,\, \left\lceil \log_2\left( \frac{r_0-\ell_0}{2\epsilon_{\mathrm{tol}}} \right) \right\rceil \right\} 
\end{align} 
interval-halving updates. Thus, choosing $K_{\max}\ge K_{\mathrm{req}}$ in Algorithm~\ref{alg:bisec} guarantees that the prescribed scalar-root accuracy is reached in exact arithmetic. The procedure may terminate earlier if an exact zero of $h$ is found. If $r_0=\ell_0$, the root is already determined and no bisection updates are required.

The above analysis establishes the unique NE of the continuous game and convergence of its computation for a fixed set of participating organizations within a given FL training round. Algorithm~\ref{alg:bisec} subsequently applies integer post-processing and verifies IR at the resulting integer allocation by checking $U_n(\dbold)\ge0$ for every $n\in\mathcal{N}$. Particularly, integer post-processing selects a feasible allocation by comparing floor and ceiling candidates through the potential objective. A single sequential pass does not establish an NE of the integer game. IR is evaluated at the final integer allocation, and an allocation that fails this check is reported as infeasible.
}


\begin{table}[tbp]
\centering
\caption{Computational Complexity and CPU Runtime.}
\label{tab:runtime_complexity}
\renewcommand{\arraystretch}{1.15}
\resizebox{0.95\columnwidth}{!}{%
\begin{tabular}{lccc}
\toprule
\rowcolor[HTML]{EFEFEF} 
\textbf{Scheme} 
  & \textbf{\begin{tabular}[c]{@{}c@{}}Preprocessing\\ Complexity\end{tabular}} 
  & \textbf{\begin{tabular}[c]{@{}c@{}}Strategy\\ Computation\end{tabular}} 
  & \textbf{\begin{tabular}[c]{@{}c@{}}Solver Time\\ (ms)\end{tabular}} \\ \hline
\midrule
\proposed 
  & $\mathcal{O}(N^2)$ 
  & $\mathcal{O}(N(1+K_m))$ 
  & $1.506 \pm 0.187$ \\
\wpr 
  & $\mathcal{O}(N^2)$ 
  & $\mathcal{O}(N(1+K_m))$ 
  & $1.440 \pm 0.146$ \\
\wco 
  & $\mathcal{O}(N)^*$ 
  & $\mathcal{O}(N(1+K_m))$ 
  & $1.436 \pm 0.114$ \\
\begin{tabular}[l]{@{}l@{}}\vcfl, \wdg,\\\radg, \madg,\\Fixed-half\end{tabular}
  & Not required 
  & $\mathcal{O}(N)$ 
  & Not timed \\
\bottomrule
\end{tabular}%
}

\vspace{3pt}
\begin{minipage}{\columnwidth}
\scriptsize
$K_m$ is the number of bisection iterations for method $m$. Reported as mean $\pm$ std across 810 runs.

$^*$For WCO, setting $\theta_n=0$ reduces preprocessing to $\mathcal{O}(N)$ (computing $c_n$ at line~\ref{alg:c_n} in Algorithm~\ref{alg:bisec}).
\end{minipage}
\end{table}

{{
\subsection{Comparative complexity analysis and Runtime measurement}
\label{proof:runtime}
\paragraph{Theoretical complexity} For Algorithm~\ref{alg:bisec}, computing the residual competition coefficients during parameter initialization in line~\ref{alg:theta} requires $\mathcal{O}(N^2)$ operations. Each bisection iteration evaluates the candidate data generation strategies and the resulting global loss in $\mathcal{O}(N)$ operations. Sequential integer post-processing starting from line~\ref{alg:int_round} and utility evaluation at the final integer allocation for IR verification starting from line~\ref{alg:ir_ver} each require at most $\mathcal{O}(N^2)$ operations. Consequently, the complete procedure has complexity $\mathcal{O}(N^2+NK)$, where $K$ is the number of bisection iterations

\paragraph{Analytical and experimental runtime compared with other baselines}\label{para:experimental_convergence} We compare this procedure of {\proposed} with several baseline methods. {\wpr} uses the NE solver with $\xi=0$, while {\wco} uses it with competition disabled (i.e., $\gamma=0$). \vcfl, \wdg, \radg, \madg, and Fixed-half construct their data generation profiles directly and require no NE optimization. Comparisons are presented in Table~\ref{tab:runtime_complexity}.
}}

{{
The runtime evaluation covers 810 games: three datasets, three heterogeneity levels, three competition levels, $\epsilon_{\text{tol}}=10^{-12}$ and 30 random seeds per setting, with $N=10$ and $\xi=90$ for {\proposed}. All measurements were conducted on a Windows 11 system equipped with an Intel Core i7-1365U CPU using a single Python process. {\proposed} converged in all 810 games, requiring $37.6\pm0.7$ bisection iterations on average. Its mean solver runtime was approximately $1.5$ ms, close to those of {\wpr} and {\wco}.

\paragraph{Computational scaling of {\proposed}} As shown in Table~\ref{tab:runtime_scalability}, we evaluated {\proposed} scaling from $N=10$ to $N=2000$ on CIFAR-10 with $\alpha_D=0.5$ and $\bar\gamma$ drawn from $\mathcal{U}(0, 1)\times2\times 10^{-3}$, with $30$ random seeds at each $N$. $\psi_n$ is multiplied by $N/10$ and $\bar\gamma_{n, n'}$ is multiplied by $9/(N-1)$. This is because, as shown in Eq.~\eqref{eq:solution}, ${\psi_n \epshat}/{N\varrho}$ shrinks as $N$ grows, while $\theta_n$ grows roughly like $N$. The rescaling helps maintaining nontrivial data generation decisions across system sizes while assessing the computational scaling. The convergence iteration count remained approximately 37 across these system sizes. Although the solver remained below $5$ ms on average, \textit{parameter initialization}  in Algorithm~\ref{alg:bisec} becomes the dominant cost at $N=2000$. These results demonstrate low proposed solver latency in the evaluated settings and identify \textit{parameter initialization step} as the principal computational limitation at larger scales.
}}
\begin{table}[t!]
\centering
\caption{Computational scaling evaluated on {\proposed}.}
\label{tab:runtime_scalability}
\small
\renewcommand{\arraystretch}{1.15}
\setlength{\tabcolsep}{4pt}
\begin{tabular}{ccc}
\toprule
\rowcolor[HTML]{EFEFEF}
$\boldsymbol{N}$
& \shortstack{\textbf{Parameter}\\\textbf{preprocessing (ms)}}
& \shortstack{\textbf{Solver time (ms)}} \\
\midrule \hline
10
& $0.108 \pm 0.003$
& $1.453 \pm 0.029$ \\

100
& $0.139 \pm 0.016$
& $1.659 \pm 0.146$ \\

1000
& $4.741 \pm 0.317$
& $3.155 \pm 0.113$ \\

2000
& $34.412 \pm 1.356$
& $4.431 \pm 0.187$ \\
\bottomrule
\end{tabular}

\end{table}

{
\subsection{Parameter estimation and Analytical sensitivity}
\label{subsec:param_est}
This subsection describes practical approaches for estimating the economic and scaling-law-based learning-performance parameters and analyzes their effects on the marginal organizational utility of synthetic-data generation. We also clarify which parameters are fitted from learning experiments and which are specified as experimental settings.

\noindent \underline{\textbf{Practical parameter estimation}:} 

\paragraph{Competition intensity $\gamma_{n,n'}$} Competition intensity, defined in Def.~\ref{def:competition_intensity}, can be informed by \textit{the overlap between organizations’ product categories or target customer segments}. For example, the ratio of shared market segments to the combined set of segments provides a symmetric proxy in $[0,1]$, consistent with $\gamma_{n,n'}=\gamma_{n',n}$. Applying this proxy requires a common definition of the relevant market and segment granularity. In our experiments, the competition levels specified in Table~\ref{tab:parameters}  are scenario parameters used to examine different competitive conditions, rather than an empirically calibrated market measure.

\paragraph{The contribution-gap coefficient $\phi$} This coefficient, defined in Eq.~\eqref{eq:coopetition_loss_1}, converts a competition-weighted local-loss gap into its modeled economic value. Once $\gamma_{n,n'}$ is specified, $\phi_{n'}$ could be estimated from pilot observations  or business assessments of the corresponding economic effects~\cite{dalessandro2014bigger}. For organization $n'$, we have $\gamma_{n,n'}[\epsilon_{n'}-\epsilon_n]$ denote an observed competition-weighted local-loss gap, and let $V_{n,n'}$ denote its estimated economic value. A possible calibration is as follows:
\begin{align}
    \phi_{n'} \approx \frac{V_{n,n'}}{\gamma_{n, n'} \left[\esilondgenother {-} \esilondngen\right]},
\end{align}
In our experiments, we treat $\phi_{n'}$ as an economic scenario parameter and analyze its influence jointly with competition intensity and compensation through $\theta_n$ (which we also analyze in the analytical sensitivity section below), defined in Eq.~\eqref{eq:theta_and_c}.
\begin{align}
    \theta_n = \sum_{n' \neq n} \gamma_{n, {n'}} (\phi_{n'} - \xi).
\end{align}
Besides, we also have an assumption of redistribution constraints $\xi\le\min_{n'}\phi_{n'}$. Particularly, the stability of the redistribution mechanism requires that the compensation rate $\xi$ not exceed the per-unit revenue $\phi_{n'}$ that any organization derives from its rival’s contribution (see Remarks~\ref{rem:1} and~\ref{rem:contrast} with their corresponding examples).

\paragraph{Scaling-law parameters} $\alpha$, $\beta$, $\delta$ can be estimated by fitting empirical loss–data-volume curves~\cite{ hestness2017deep, wang2020machine, li2025misfitting, li2024filling} for a given learning task, data-heterogeneity setting, and fixed generator configuration. This gives observations of one-point $(d_k, \ell_k)$ on the empirical learning curve, where $d_k$ is the total training-data volume for each $k$, including local and synthetic data, and  $\ell_k$ is the corresponding measured local loss. We estimated these parameters using constrained nonlinear least squares:
\begin{align}
(\widehat{\alpha},\widehat{\beta},\widehat{\delta})
=
\underset{\substack{
\alpha>0,\;\beta>0\\
\delta \geq 0
}}
{\operatorname{argmin}}
\sum_{k=1}^{K}
\left[
\ell_k-\left(\alpha d_k^{-\beta}-\delta\right)
\right]^2.
\end{align}
Fig.~\ref{fig:hyperparameter}  presents the empirical and fitted curves for the evaluated datasets (i.e., FMNIST for easy learning tasks, CIFAR-10 for moderate learning tasks, and CIFAR-100 for difficult learning tasks) and heterogeneity settings. Thus, these parameters can be reused for learning tasks with equivalent difficulties.

\setcounter{paragraph}{0}
\noindent \underline{\textbf{Analytical sensitivity}:} 
\paragraph{Marginal utility of synthetic-data generation} We first make the dependence of organization $n$’s utility $U_n$ on synthetic-data generation $\dgen$ explicit. Holding the other organizations’ data generation decisions fixed, $U_n$ in Eq.~\eqref{eq:client_utility_one_round}  can be written as:
\begin{align}
U_n = K_n-\psi_n\epsilon+\theta_n\epsilon_n -c_n d_n^{\mathrm{gen}},
\end{align}
where $K_n$ is independent of $\dgen$, and ${\theta_n} = \sum_{n' \neq n} \gamma_{n, {n'}} (\phi_{n'} - \xi), \; c_n = \kappa_nC_n^{\text{cmp}} (\eta_n  +\mu_n) f_n^2$ (see Eq.~\eqref{eq:theta_and_c}).

Then, we have $d_n = \dloc + \dgen$, $\epsilon = \esilontotaldgen$ and define $g_n = \frac{\alpha \beta}{d_n^{\beta+1}}, B_n=\frac{\psi_n \epsilon}{N \varrho}$. As a result, we have the following marginal organizational utility if synthetic-data generation.
\begin{align}
    m_n = \frac{\partial U_n}{\partial \dgen} = g_n (B_n - \theta_n) -c_n.
\end{align}
Thus, \textit{\textbf{a feasible marginal increase in data generation improves the organizational utility}} when
\begin{align}
    m_n > 0 & \iff \frac{\partial U_n}{\partial \dgen} = g_n (B_n - \theta_n) -c_n > 0 \nonumber \\
    & \iff \frac{\psi_n \epsilon}{N \varrho} > \theta_n + \frac{c_n d_n^{\beta+1}}{\alpha \beta}
\end{align}

This condition states that the marginal learning benefit $B_n=\frac{\psi_n \epsilon}{N \varrho}$ must exceed the residual competition penalty $\theta_n$ and computational cost $c_n$. \textit{Positive marginal utility $m_n > 0$ creates an incentive to generate additional data}. It also explains why the parameters must be considered jointly. For example, the effect of competition intensity depends on the uncompensated contribution-gap coefficient.

At a fixed data allocation, the economic parameter sensitivities for the marginal organizational utility $m_n$ are shown in the Table~\ref{tab:param_sens} below.
\begin{table}[t!]
\centering
\caption{Economic parameter sensitivities.}
\label{tab:param_sens}
\small
\renewcommand{\arraystretch}{1.3}
\setlength{\tabcolsep}{4pt}
\begin{tabular}{ll}
\toprule
\rowcolor[HTML]{EFEFEF}
\textbf{Parameter} & \textbf{Sensitivity} \\
\midrule \hline
Competition intensity $\gamma_{n,{n'}}$ & $\displaystyle \frac{\partial m_n}{\partial \gamma_{n,{n'}}} = -g_n(\phi_{n'}-\xi)$ \\[6pt]
Contribution-gap coefficient $\phi_{n'}$ & $\displaystyle \frac{\partial m_n}{\partial \phi_{n'}} = -g_n \gamma_{n,{n'}}$ \\[6pt]
Compensation rate $\xi$ & $\displaystyle \frac{\partial m_n}{\partial \xi} = g_n \sum_{n' \neq n} \gamma_{n, {n'}}$ \\[6pt]
Marginal computational cost $c_n$ & $\displaystyle \frac{\partial m_n}{\partial c_n} = -1$ \\[6pt]
Learning-revenue coefficient $\psi_n$ & $\displaystyle \frac{\partial m_n}{\partial \psi_n} = \frac{g_n \epsilon}{N \varrho}$ \\
\bottomrule
\end{tabular}
\end{table}
\textbf{As a result, under the assumed condition $\xi\le\min_{n'}\phi_{n'}$, stronger competition $\gamma_{n,{n'}}$ and larger contribution-gap coefficients $\phi_{n'}$ reduce the marginal utility or demotivate organizations to generate synthetic data. Larger compensation rate $\xi$ increases it, while higher computational costs $c_n$ reduce it and a greater economic value $\psi_n$ assigned to learning improvements increases it.}

To characterize sensitivity to the scaling-law learning parameters $\alpha,\beta$ and $\delta$, we define 
\begin{align}
    Q_{\beta} = \frac{1}{N}\sum_{n=1}^{N} d_n^{-\beta}, \;\; T_{\beta} = \frac{1}{N}\sum_{n=1}^{N} d_n^{-\beta}\ln{d_n}
\end{align}
Then, Eq.~\eqref{eq:data_size_loss}  can be rewritten as:
\begin{align}
    \epsilon = \text{exp}\left(\frac{\alpha Q_{\beta} - \delta - 1}{\varrho}\right),
\end{align}
and, holding the data-generation profile and other parameters unchanged, we have the following expressions:
\begin{align}
    & \frac{\partial m_n}{\partial \alpha} = \frac{g_n}{\alpha}(B_n-\theta_n) + \frac{g_n B_n Q_\beta}{\varrho}, \nonumber \\
    & \frac{\partial m_n}{\partial \beta} = g_n\left(\frac{1}{\beta}-\ln d_n\right)(B_n-\theta_n) - \frac{g_nB_n\alpha T_\beta}{\varrho}, \nonumber \\
    & \frac{\partial m_n}{\partial \delta} = -\frac{g_nB_n}{\varrho}. 
\end{align}

These derivatives describe how the utility gained from additional generation changes with the learning parameters. Particularly, \textit{\textbf{we have the following insights}}:
    \begin{itemize}
        \item Increasing $\alpha$ increases the marginal utility when $B_n > \theta_n$, whereas increasing $\delta$ decreases it.
        \item The effect of $\beta$ on the marginal utility depends on the data volumes $d_n$ and the balance between learning benefits $B_n$ and residual competition $\theta_n$.  
    \end{itemize}
\paragraph{Interpretation at an interior best response} 
At an interior best response (i.e., $m_n=0$), we have an exact balance of the marginal learning benefit with residual competition penalty and computational cost, as presented by:
\begin{align}
    g_n B_n = g_n \theta_n + c_n.
\end{align}
Furthermore, when we continue to evaluate at such a point, we have 
\begin{align}
    \frac{\partial m_n}{\partial \dgen} = -\frac{(\beta+1)c_n}{d_n} -\frac{B_n g_n^2}{N \varrho} < 0.
\end{align}
As a result, the marginal utility decreases as data generation increases through an interior best response. \textit{Additional synthetic data can therefore continue to reduce the learning loss while providing no further utility improvement once its marginal learning benefit is balanced by the competition loss and computational costs}.
\textbf{This explains why an organization should optimize data generation volume} $\dgen$. The experiments in Section~\ref{sec:performance_evaluation} complement our analysis by evaluating data generation decisions and social welfare across competition levels, data-heterogeneity settings, and compensation rates. For example, in Table~\ref{tab:social_welfare_efficiency}, we show that {\proposed} achieves higher mean social welfare than {\wpr}, \wdg, {\radg}, {\madg}, and Fixed-half across settings.

}

\bibliographystyle{IEEEtran}
\bibliography{Refs}





\end{document}